\documentclass[11pt,english]{article}
\usepackage{mathpazo}
\usepackage[T1]{fontenc}
\usepackage[latin9]{inputenc}
\usepackage{geometry}
\usepackage{babel}

\usepackage{float}
\usepackage{calc}
\usepackage{mathtools}
\usepackage{amsmath}
\usepackage{amsthm}
\usepackage{amssymb}
\usepackage{color}
\usepackage{enumitem}
\usepackage{xargs}[2008/03/08]
\usepackage[colorlinks,
    linkcolor=red,
    anchorcolor=blue,
    citecolor=blue
    ]{hyperref}
\usepackage{appendix}
\allowdisplaybreaks

\makeatletter

\floatstyle{ruled}
\newfloat{algorithm}{tbp}{loa}
\providecommand{\algorithmname}{Algorithm}
\floatname{algorithm}{\protect\algorithmname}

\theoremstyle{plain}
\newtheorem{thm}{\protect\theoremname}
\theoremstyle{remark}

\theoremstyle{plain}
\newtheorem{cor}{Corollary}
\theoremstyle{plain}

\theoremstyle{plain}
\newtheorem{lem}{Lemma}
\theoremstyle{plain}
\newtheorem{fact}{Fact}

 \newcommand{\citep}{\cite}
 \newcommand{\citet}{\cite}

\usepackage{caption}
\usepackage{subcaption}

\usepackage{algorithm}
\usepackage{algorithmicx}
\usepackage{algpseudocode}
\usepackage{cleveref}

\newcommand{\algmargin}{\the\ALG@thistlm}

\newlength{\whilewidth}
\algdef{SE}[parWHILE]{parWhile}{EndparWhile}[1]
{\parbox[t]{\dimexpr\linewidth-\algmargin}{%
		\hangindent\whilewidth\strut\algorithmicwhile\ #1\ \algorithmicdo\strut}}{\algorithmicend\ \algorithmicwhile}%
\algnewcommand{\parState}[1]{\State%
	\parbox[t]{\dimexpr\linewidth-\algmargin}{\strut #1\strut}}

\makeatother

\providecommand{\theoremname}{Theorem}

\title{Risk-Sensitive Reinforcement Learning: Near-Optimal Risk-Sample Tradeoff in Regret}

\author{%
    \normalsize Yingjie Fei\thanks{School of Operations Research and Information Engineering, Cornell University; \texttt{yf275@cornell.edu}}
    \qquad
    \normalsize Zhuoran Yang\thanks{Department of Operations Research and Financial Engineering, Princeton University; \texttt{zy6@princeton.edu}} 
    \qquad
    \normalsize Yudong Chen\thanks{School of Operations Research and Information Engineering, Cornell University; \texttt{yudong.chen@cornell.edu}}
    \qquad
    \normalsize Zhaoran Wang\thanks{Department of Industrial Engineering and Management Sciences, Northwestern University; \texttt{zhaoranwang@gmail.com}}
    \qquad
    \normalsize Qiaomin Xie\thanks{School of Operations Research and Information Engineering, Cornell University; \texttt{qiaomin.xie@cornell.edu}}
}
\date{}

\begin{document}
\global\long\def\cC{{\cal C}}%
\global\long\def\cM{{\cal M}}%
\global\long\def\cN{{\cal N}}%
\global\long\def\cV{{\cal V}}%
\global\long\def\cF{{\cal F}}%
\global\long\def\cR{{\cal R}}%
\global\long\def\cP{{\cal P}}%
\global\long\def\cG{{\cal G}}%
\global\long\def\cB{{\cal B}}%
\global\long\def\cD{{\cal D}}%
\global\long\def\cX{{\cal X}}%
\global\long\def\cS{{\cal S}}%
\global\long\def\cA{{\cal A}}%
\global\long\def\cY{{\cal Y}}%
\global\long\def\cL{{\cal L}}%
\global\long\def\cQ{{\cal Q}}%
\global\long\def\cU{{\cal U}}%
\global\long\def\real{\mathbb{R}}%
\global\long\def\E{\mathbb{E}}%
\global\long\def\P{\mathbb{P}}%
\global\long\def\pol{\pi}%
\global\long\def\indic{\mathbb{I}}%
\global\long\def\Ps{P^{\star}}%
\global\long\def\Php{\hat{P}^{\pol}}%
\global\long\def\Pp{P^{\pol}}%
\global\long\def\Vs{V^{\star}}%
\global\long\def\Vp{V^{\pol}}%
\global\long\def\Qs{Q^{\star}}%
\global\long\def\Qp{Q^{\pol}}%
\global\long\def\vhat{\hat{v}}%
\global\long\def\zhat{\hat{z}}%
\global\long\def\e{\mathbf{e}}%
\global\long\def\g{\mathbf{g}}%
\global\long\def\w{w}%
\global\long\def\v{\mathbf{v}}%
\global\long\def\fmap{\phi}%
\global\long\def\pmap{\mu}%
\global\long\def\IdMat{I}%
\global\long\def\bP{\mathbf{P}}%
\global\long\def\A{\mathbf{A}}%
\global\long\def\ucbrate{\lambda_{\text{UCB}}}%
\global\long\def\tmp{\beta}%
\global\long\def\diag{\mathop{\text{diag}}}%
\global\long\def\argmin{\mathop{\text{argmin}}}%
\global\long\def\argmax{\mathop{\text{argmax}}}%
\newcommandx\norm[2][usedefault, addprefix=\global, 1=\#1]{\ensuremath{\left\Vert #1\right\Vert {}_{#2}}}%
\global\long\def\conf{\text{conf}}%
\global\long\def\lse{\mathsf{lse}}%
\renewcommandx\norm[2][usedefault, addprefix=\global, 1=\#1]{\Vert#1\|_{#2}}%
\global\long\def\reg{\textup{Regret}}%
\global\long\def\Var{\textup{Var}}%

\maketitle

\begin{abstract}
  We study risk-sensitive reinforcement learning in episodic Markov
  decision processes with unknown transition kernels, where the goal is to optimize the total reward
  under the risk measure of exponential utility. We propose two provably
  efficient model-free algorithms, Risk-Sensitive Value Iteration (RSVI) and Risk-Sensitive Q-learning (RSQ). 
  These algorithms implement
  a form of risk-sensitive optimism in the face of uncertainty, which
  adapts to both risk-seeking and risk-averse modes of exploration.
  We prove that RSVI attains an $\ensuremath{\tilde{O}\big(\lambda(|\beta| H^2) \cdot \sqrt{H^{3} S^{2}AT} \big)}$
  regret, while RSQ attains an $\ensuremath{\tilde{O}\big(\lambda(|\beta| H^2) \cdot \sqrt{H^{4} SAT} \big)}$
  regret, where $\lambda(u) = (e^{3u}-1)/u$ for $u>0$.
  In the above, $\beta$ is the risk parameter of the exponential
  utility function, $S$ the number of states, $A$ the number of actions,  $T$ the total number of timesteps, and $H$ the episode length.
  On the flip side, we establish a regret lower bound showing that the exponential
  dependence on $|\beta|$ and $H$ is unavoidable for
  any algorithm with an $\tilde{O}(\sqrt{T})$ regret (even when the risk objective is on the same scale as the original reward), thus
  certifying the near-optimality of the proposed algorithms. Our results
  demonstrate that incorporating risk awareness into reinforcement learning necessitates an exponential cost in $|\beta|$ and $H$,
  which quantifies the fundamental tradeoff between risk sensitivity
  (related to aleatoric uncertainty) and sample efficiency (related
  to epistemic uncertainty). To the best of our knowledge, this is the  first regret analysis of risk-sensitive reinforcement learning with
  the exponential utility.
\end{abstract}

\section{Introduction}

Risk-sensitive reinforcement learning (RL) concerns
learning to act in a dynamic environment while taking into account
risks that arise during the learning process. Effective management
of risks in RL is critical to many real-world applications
such as autonomous driving \citep{garcia2015comprehensive}, real-time
strategy games \citep{vinyals2019grandmaster}, financial investment
\citep{markowitz1952portfolio}, etc. In neuroscience, risk-sensitive RL has been applied to model human behaviors in decision making \citep{niv2012neural, shen2014risk}.

In this paper, we consider risk-sensitive RL with the exponential
utility \citep{howard1972risk} under episodic Markov decision processes (MDPs) with unknown transition kernels. Informally,
the agent aims to maximize a risk-sensitive objective function of the form 
\begin{equation}
V=\frac{1}{\beta}\log\left\{ \E e^{\beta R}\right\} ,\label{eq:risk_sen_obj_informal}
\end{equation}
where $R$ is the total reward the agent receives, and $\beta\ne0$ is a real-valued parameter that controls risk preference of the agent; see Equation \eqref{eq:value_func} for a formal
definition of $V$. The objective $V$ admits the Taylor expansion
$V = \E[R]+\frac{\beta}{2}\Var(R)+O(\beta^{2}).$ It can be seen
that for $\beta>0$ the agent is risk-seeking (favoring high uncertainty
in $R$), for $\beta<0$ the agent is risk-averse (favoring low
uncertainty in $R$), and a larger $|\beta|$ implies
higher risk-sensitivity. When $\beta\to0$, the agent
tends to be risk-neutral and the objective reduces to the expected reward objective $ V=\E[R] $ standard in RL. Therefore, the risk-sensitive objective
in~\eqref{eq:risk_sen_obj_informal} covers the entire spectrum of risk sensitivity by varying $\beta$. In addition, the formulation
\eqref{eq:risk_sen_obj_informal} is closely related to RL with constraints.
For example, a negative risk parameter $\beta$ controls
the tail of a risk distribution so as to mitigate the chance of receiving
a total reward $R$ that is excessively low. We refer to \citet[Section 2.1]{maillard2013robust}
for an in-depth discussion of this connection. 

The challenge of risk-sensitive RL lies both in the non-linearity of the objective function and in designing a risk-aware exploration mechanism.
In particular, as we elaborate in Section \ref{sec:bellman}, the non-linear objective
function \eqref{eq:risk_sen_obj_informal} induces a non-linear Bellman equation.
Classical RL algorithms are inappropriate in this setting,
as their design crucially relies on the linearity of Bellman equations.
On the other hand, effective exploration has been well known to be crucial to RL algorithm design, yet it is not clear how to design an algorithm
that efficiently explores  uncertain environments while at the same time adapting to 
the risk-sensitive
objective \eqref{eq:risk_sen_obj_informal} of agents with different risk parameter $\beta$.

To address these difficulties, we propose two model-free algorithms, 
Risk-Sensitive Value Iteration (RSVI) and Risk-Sensitive Q-learning (RSQ). Specifically, RSVI is a batch algorithm and RSQ is an online algorithm; both families of batch and online algorithms see broad applications in practice.
We demonstrate in Section \ref{sec:algorithm} that our proposed algorithms implement a form of risk-sensitive optimism for exploration. 
Importantly, the exact implementation of optimism depends
on both the magnitude and the sign of the risk parameter, and therefore
applies to both risk-seeking and risk-averse modes of learning.  
Letting $\lambda(u) = (e^{3u}-1)/u$ for $u>0$, we prove that RSVI attains an 
$\ensuremath{\tilde{O}\big(\lambda(|\beta| H^2) \cdot \sqrt{H^{3} S^{2}AT} \big)}$
regret, and RSQ achieves an 
$\ensuremath{\tilde{O}\big(\lambda(|\beta| H^2) \cdot \sqrt{H^{4} SAT} \big)}$
 regret. Here, $S$ and $A$ are the numbers of states and actions, respectively,
$T$ is the total number of timesteps, and $H$ is the length of each episode.
These regret bounds interpolate across different regimes of risk sensitivity and subsume  existing results under the risk-neutral setting. Compared with risk-neutral RL (corresponding to $\beta \rightarrow  0$), our general regret bounds feature an exponential dependency on $|\beta|$ and $H$, even though the risk-sensitive objective \eqref{eq:risk_sen_obj_informal} is on the same scale as the total reward; see Figure \ref{fig:exp_factor} for a plot of the exponential factor $\lambda(|\beta| H^2)$.
Complementarily, we prove a lower bound showing that such an exponential dependency is inevitable for any algorithm and thus certifies the near-optimality of the proposed algorithms. To the best of our knowledge, our work provides the first regret analysis of risk-sensitive RL with the exponential utility. 

Our upper and lower bounds demonstrate the fundamental tradeoff between risk sensitivity and sample efficiency in RL.\footnote{By standard arguments, regret can be translated into sample complexity bounds and vice versa; see \cite{jin2018q}.} 
Broadly speaking, risk sensitivity is associated with \emph{aleatoric} uncertainty, which originates from the inherent randomness of state transition, actions and rewards, whereas sample efficiency is associated with \emph{epistemic} uncertainty, which arises from 
imperfect knowledge of the environment/system
and can be reduced by more exploration \citep{depeweg2017decomposition, clements2019estimating}. These two notions of uncertainty are usually decoupled in the regret analysis of risk-neutral RL---in particular, using the expected reward as the objective effectively suppresses the aleatoric uncertainty. In risk-sensitive RL, we establish that there is a fundamental connection and tradeoff between these two forms of uncertainty: the risk-seeking and risk-averse regimes both incur an exponential cost in $|\beta|$ and $H$ on the regret, whereas the regret is polynomial in $ H $ in the risk-neutral regime.

\begin{figure}
  \centering
  \includegraphics[scale=.5]{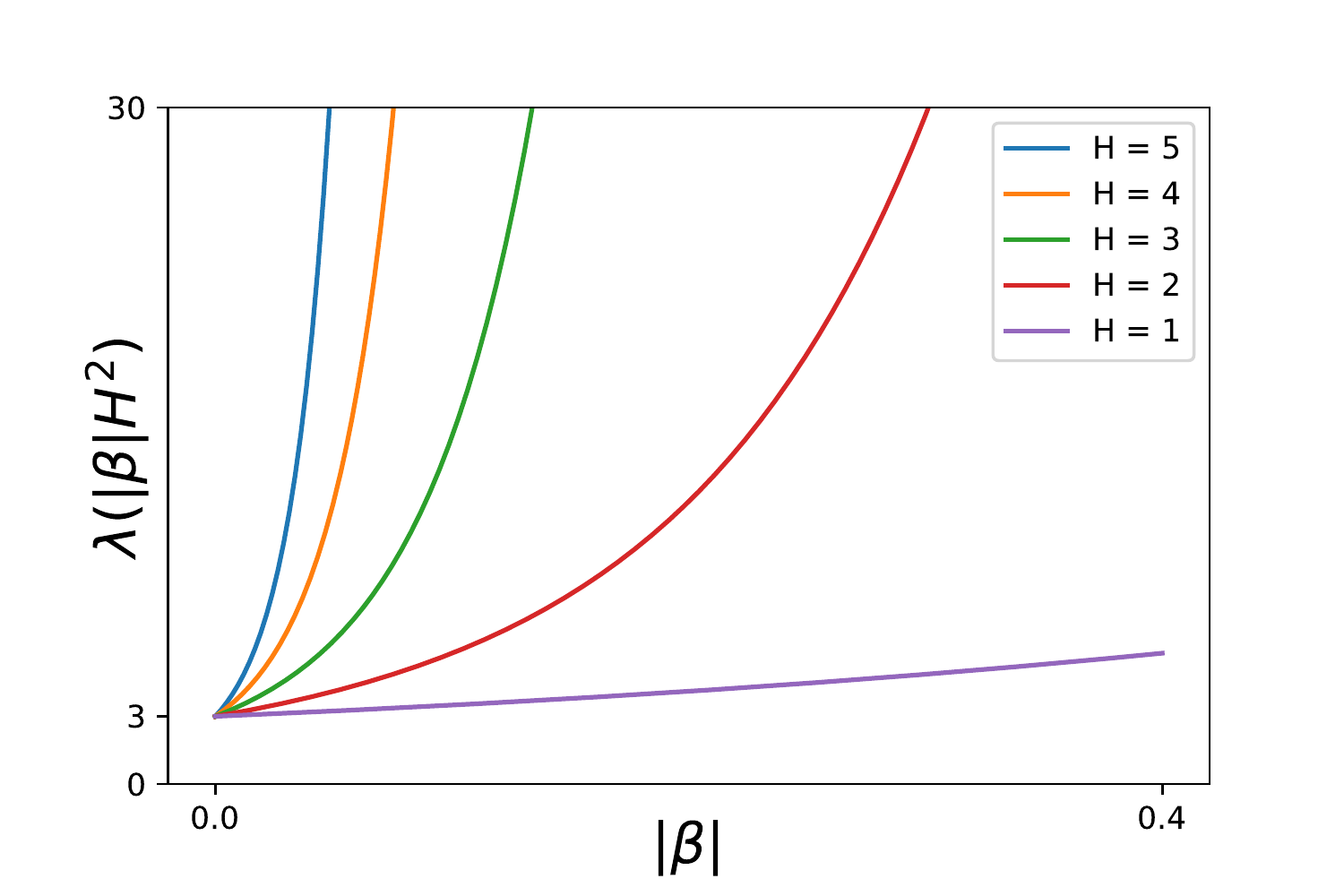}
  \caption{Scaling of $\lambda(|\beta| H^2)$ in risk sensitivity $|\beta|$ for different values of episode length $H$.}
  \label{fig:exp_factor}
\end{figure}

\paragraph*{Our contributions.}
The contributions of our work can be summarized as follows:
\begin{itemize}
	\item We consider the problem of risk-sensitive RL with the exponential utility.  
	We propose two provably efficient model-free algorithms, namely RSVI and RSQ, that implement risk-sensitive optimism in the face of uncertainty;
	\item 
	We provide regret analysis for both algorithms over the entire spectrum of risk parameter $\beta$.
	As $\beta\to 0$, we show that our results recover 
	the existing regret bounds in the risk-neutral setting;
	\item We provide a lower bound result that certifies the near-optimality of our upper bounds and reveals a fundamental tradeoff between risk sensitivity and sample complexity.
\end{itemize}

\paragraph*{Related work.}


RL with risk-sensitive utility functions have been studied in several work. The work \citet{mihatsch2002risk} proposes TD(0) and Q-learning-style algorithms
that transform temporal differences instead of cumulative rewards, and proves their convergence.
Risk-sensitive RL with a general family of utility functions is studied
in \citet{shen2014risk}, which also proposes a Q-learning algorithm with convergence guarantees. The work of \citet{eriksson2019epistemic}
studies a risk-sensitive policy gradient algorithm, though with no theoretical guarantees. We remark that while substantial work has been devoted to designing risk-sensitive RL algorithms and proving their convergence, the issues of exploration, sample efficiency and regret bounds have rarely been studied. Our work narrows this gap in the literature by studying regret bounds of model-free algorithms for risk-sensitive RL.

The exponential utility has also been been investigated
in the more classical setting of MDPs.  
Following the seminal work of \citet{howard1972risk},
this line of work includes \citet{bauerle2014more,borkar2001sensitivity,borkar2002q,borkar2002risk,cavazos2000vanishing,coraluppi1999risk,di1999risk,fernandez1997risk,fleming1995risk,hernandez1996risk,marcus1997risk,osogami2012robustness,shen2013risk,whittle1990risk}. Note that these papers impose more restrictive assumptions and study different types of results than ours. Specifically, they assume known transition kernels or access to simulators, and they do not conduct finite-time or finite-sample analysis.
Another related direction to ours is RL with risk/safety
constraints studied by \citet{achiam2017constrained,altman1999constrained,chow2014algorithms,chow2015risk,chow2017risk,chow2019lyapunov,ding2020provably,efroni2020exploration,qiu2020upper,tamar2015policy,xie2018block,zheng2020constrained}, and readers are also referred to \citet{fu2018risk} for an excellent
survey on this topic. Compared to our work, that line of work focuses on constrained RL problems with  different risk criteria. Other related problems include  risk-sensitive games  \citet{basu2012zero,basu2014zero,bauerle2017zero,cavazos2019vanishing,huang2019model,jaskiewicz2014stationary,klompstra2000nash,wei2019nonzero}, and risk-sensitive bandits \citet{cassel2018general,denardo2007risk,denardo2013multi,maillard2013robust,sani2012risk,sun2017safety,vakili2016risk,yu2013sample,zimin2014generalized}. Bandit problems are special cases of the RL problem that we investigate, with both the number of states and episode length being equal to one. As such, both our settings and results are more general than those obtained in bandit problems.

\paragraph*{Notations.}

For a positive integer $n$, let $[n]\coloneqq\{1,2,\ldots,n\}$.
For two non-negative sequences $\{a_{i}\}$ and $\{b_{i}\}$, we
write $a_{i}\lesssim b_{i}$ if there exists a universal constant
$C>0$ such that $a_{i}\le Cb_{i}$ for all $i$. We write $a_{i}\asymp b_{i}$
if $a_{i}\lesssim b_{i}$ and $b_{i}\lesssim a_{i}$. We use $\tilde{O}(\cdot)$
to denote $O(\cdot)$ while hiding logarithmic factors.

\section{Problem setup\label{sec:setup}}
\subsection{Episodic MDPs and risk-sensitive objective}
We consider the setting of episodic MDPs, denoted
by $\text{MDP}(\cS,\cA,H,\cP,\cR)$, where $\cS$ is the set of possible
states, $\cA$ is the set of possible actions, $H$
is the length of each episode, and $\cP=\{P_{h}\}_{h\in[H]}$ and
$\cR=\{r_{h}\}_{h\in[H]}$ are the sets of state transition kernels and reward functions, respectively. In particular, for each $h\in[H]$, $P_{h}(\cdot\,|\,s,a)$
is the distribution of the next state if action $a$ is taken
in state $s$ at step $h$. We assume that $\cS$ and $\cA$ are finite discrete spaces, and let $S=\left|\cS\right|$ and $A=\left|\cA\right|$ denote their cardinalities. We assume that the agent does not have access to $ \{P_{h}\} $ and  that each $ r_{h}:\cS\times\cA\to[0,1]  $
is a deterministic function.

An agent interacts with an episodic MDP as follows. At the beginning
of each episode, an initial state $s_{1}$ is chosen arbitrarily by
the environment. In each step $h\in[H]$, the agent observes
a state $s_{h}\in\cS$, chooses an action $a_{h}\in\cA$, and receives
a reward $r_{h}(s_{h},a_{h})$. The MDP then transitions into
a new state $s_{h+1} \sim P_{h}(\cdot\,|\,s_{h},a_{h})$. We use the convention that the episode terminates when a state $s_{H+1}$
at step $H+1$ is reached, at which the agent does not take an action and receives no reward.

A policy $\pi=\{\pi_{h}\}_{h\in[H]}$ of an agent is a sequence of functions $\pi_{h}:\cS\to\cA$, where $\pi_{h}(s)$
is the action that the agent takes in state $s$ at step $h$ of an
episode. For each $h\in[H]$, we define the value function $V_{h}^{\pi}:\cS\to\real$
of a policy $\pi$ as the expected value of cumulative rewards the
agent receives under a risk measure of exponential utility by executing
policy $\pi$ starting from an arbitrary state at step $h$. Specifically,
we have 
\begin{equation}
V_{h}^{\pi}(s)\coloneqq\frac{1}{\tmp}\log\left\{ \E\left[\exp\left(\tmp\sum_{h=1}^{H}r_{h}(s_{h},\pi_{h}(s_{h}))\right)\  \Bigg|\  s_h=s \right]\right\} ,\label{eq:value_func}
\end{equation}
for each $(h,s)\in[H]\times\cS$. Here $\beta\ne0$ is the risk parameter
of the exponential utility: $\beta>0$
corresponds to a risk-seeking value function, $\beta<0$ corresponds
to a risk-averse value function, and as $\beta\to0$ the agent tends
to be risk-neutral and we recover the classical value function $ V_{h}^{\pi}(s) = \E[ \sum_{h=1}^{H}r_{h}(s_{h},\pi_{h}(s_{h})) \ |\  s_h=s ] $ in
RL.
The goal of the agent is to find a policy $\pi$ such that $V_{1}^{\pi}(s)$
is maximized for all state $s\in\cS$. 
Note the logarithm and rescaling by $ 1/\beta $ in the above definition, which puts the objective $V_{1}^{\pi}(s)$ on the same scale as the total reward; this scaling property is made formal in Lemma~\ref{lem:simple_bound} below.

\subsection{Bellman equations and regret} \label{sec:bellman}

We further define the action-value function $Q_{h}^{\pi}:\cS\times\cA\to\real$,
which gives the expected value of the risk measured by the exponential utility when the agent starts from an arbitrary state-action
pair at step $h$ and follows policy $\pi$ afterwards; that is, 
\[
Q_{h}^{\pi}(s,a)\coloneqq\frac{1}{\tmp}\log\left\{ \exp(\tmp\cdot r_{h}(s,a))\E\left[\exp\left(\tmp\sum_{h'=h+1}^{H}r_{h'}(s_{h'},a_{h'})\right)\Bigg|\ s_{h}=s,a_{h}=a\right]\right\} ,
\]
for all $(h,s,a)\in[H]\times\cS\times\cA$. The Bellman equation associated with policy $\pi$ is given by
\begin{equation}
\label{eq:bellman_primal}
\begin{aligned}
Q_{h}^{\pi}(s,a) & =r_{h}(s,a)+\frac{1}{\tmp}\log\left\{ \E_{s'\sim P_{h}(\cdot\,|\,s,a)}\left[\exp\left(\beta\cdot V_{h+1}^{\pi}(s')\right)\right]\right\} ,\\
V_{h}^{\pi}(s) & =Q_{h}^{\pi}(s,\pol_h(s)),\qquad V_{H+1}^{\pi}(s)=0, 
\end{aligned}
\end{equation}
which holds for all $(s,a)\in\cS\times\cA$. 

Under some mild regularity conditions,
there always exists an optimal policy $\pi^{*}$ which gives the optimal
value $V_{h}^{*}(s)=\sup_{\pi}V_{h}^{\pi}(s)$ for all $(h,s)\in[H]\times\cS$
\citep{bauerle2014more}. The Bellman optimality equation is given by
\begin{equation}
\label{eq:bellman_optimal_primal}
\begin{aligned}
Q_{h}^{*}(s,a) & =r_{h}(s,a)+\frac{1}{\tmp}\log\left\{ \E_{s'\sim P_{h}(\cdot\,|\,s,a)}\left[\exp\left(\beta\cdot V_{h+1}^{*}(s')\right)\right]\right\} ,\\
V_{h}^{*}(s) & =\max_{a\in\cA}Q_{h}^{*}(s,a),\qquad V_{H+1}^{*}(s)=0.
\end{aligned}
\end{equation}
This equation implies that the optimal policy $\pi^{*}$ is the greedy policy
with respect to the optimal action-value function $\{Q_{h}^{*}\}_{h\in[H]}$.
Hence, to find the optimal policy $\pi^{*}$, it suffices to estimate
the optimal action-value function. We note that both Bellman equations
\eqref{eq:bellman_primal} and \eqref{eq:bellman_optimal_primal}
are non-linear in the value and action-value functions due to non-linearity
of the exponential utility. This is in contrast with their linear risk-neutral
counterparts. 

Under the episodic MDP setting, the agent aims to learn the
optimal policy by interacting with the environment throughout a set
of episodes. For each $k\ge1$, let us denote by $s_{1}^{k}$ the
initial state chosen by the environment and $\pi^{k}$ the policy
chosen simultaneously by the agent at the beginning of episode $k$.
The difference in values between $V_{1}^{\pi^{k}}(s_{1}^{k})$ and
$V_{1}^{*}(s_{1}^{k})$ measures the expected regret or the sub-optimality
of the agent in episode $k$. 
After $K$ episodes, the
total regret for the agent is 
\begin{align}
\reg(K) & \coloneqq\sum_{k\in[K]}\left[V_{1}^{*}(s_{1}^{k})-V_{1}^{\pi^{k}}(s_{1}^{k})\right].\label{eq:regret}
\end{align}

We record the following simple worst-case upper bounds on the value functions and regret.
\begin{lem}
\label{lem:simple_bound}
	For any $ (h,s,a) \in \cS \times \cA\times [H]$, policy $ \pi $ and risk parameter $ \beta \neq 0 $, we have
	\begin{equation}\label{eq:simple_bound_value}
	0 \le V_h^\pi(s)  \le H
	\quad\text{and}\quad
	0 \le Q_h^\pi(s,a) \le H.
	\end{equation} 
	Consequently, for each $ K \ge 1 $, all policy sequences $ \pi^1,\ldots,\pi^K $ and any $ \beta \neq 0 $, we have
	\begin{equation}\label{eq:simple_bound_regret}
	0 \le \textup{Regret}(K) \le KH.
	\end{equation}
\end{lem}
\begin{proof}
	Recall the assumption that the reward functions $ \{r_h\} $ are bounded in $ [0,1] $. The lower bounds are immediate by definition. For the upper bound, we have 
	$
	V_h^\pi(s) \le \frac{1}{\tmp}\log\left\{ \E\left[\exp\left(\tmp H \right) \right]\right\} = H.
	$
	Upper bounds for $ Q_h^\pi $ and the regret follow similarly.
\end{proof}

While straightforward, the above lemma highlights an important point: the risk and regret are on the same scale as the reward. In particular, the upper bounds above are \emph{independent} of $\beta$ and \emph{linear} in the horizon length $ H $---the same as in the standard MDP setting---because the $ \log $ and $ \exp $ functions in the definition of the objective  function \eqref{eq:value_func} cancel with each other in the worst case. Therefore, the exponential dependence of the regret on $|\beta|$ and $H$, which we establish below in Section \ref{sec:main}, is not merely a consequence of scaling but rather is inherent in the risk-sensitive setting.

\section{Algorithms \label{sec:algorithm}}

The non-linearity of the Bellman equations, discussed in Section \ref{sec:bellman},
creates challenges in algorithmic design. In particular,
standard model-free algorithms such as least-squares value iteration
(LSVI) and Q-learning are no longer appropriate since they specialize
to the risk-neutral setting with linear Bellman equations. In this
section, we 
present risk-sensitive LSVI and Q-learning algorithms
that adapt to both the non-linear Bellman equations and any valid risk parameter $\beta$.

\subsection{Risk-Sensitive Value Iteration}

We first present Risk-Sensitive Value Iteration (RSVI)  in Algorithm \ref{alg:alg_lsvi}. Algorithm \ref{alg:alg_lsvi} is inspired by LSVI-UCB of \citet{jin2019provably},
which is in turn motivated by the idea of LSVI \citep{bradtke1996linear,osband2014generalization}
and the classical value-iteration algorithm. Like LSVI-UCB, Algorithm
\ref{alg:alg_lsvi} applies the Upper Confidence Bound (UCB) by incorporating a bonus term to value estimates of state-action pairs, which therefore implements the principle of Optimism in the Face of Uncertainty (OFU) \citep{jaksch2010near}. 

\paragraph{Mechanism of Algorithm \ref{alg:alg_lsvi}.}

The algorithm mainly consists of the value estimation step (Line \ref{line:LSVI_estim_value_begin}--\ref{line:LSVI_estim_value_end})
and the policy execution step (Line \ref{line:LSVI_exec_policy_begin}--\ref{line:LSVI_exec_policy_end}).
In Line \ref{line:LSVI_interm_value}, the algorithm computes the
intermediate value $w_{h}$ by a least-squares update
\begin{equation}
w_{h}\leftarrow\argmin_{w\in\real^{SA}}\sum_{\tau\in[k-1]}\left[e^{\beta[r_{h}(s_{h}^{\tau},a_{h}^{\tau})+V_{h+1}(s_{h+1}^{\tau})]}-w^{\top}\fmap(s_{h}^{\tau},a_{h}^{\tau})\right]^{2}.\label{eq:LSVI_LS_step}
\end{equation}
Here, $\{(s_{h}^{\tau},a_{h}^{\tau},s_{h+1}^{\tau})\}_{\tau\in[k-1]}$
are accessed from the dataset $\cD_{h}$ for each $h\in[H]$,
and $\fmap(\cdot,\cdot)$ denotes the canonical basis in $\real^{SA}$.
Line \ref{line:LSVI_interm_value} can be efficiently implemented
by computing sample means of $e^{\beta[r_{h}(s,a)+V_{h+1}(s')]}$
over those state-action pairs that the algorithm has visited. 
Therefore, it can also be interpreted as estimating the sample means
of exponentiated $Q$-values under visitation measures induced by
the transition kernels $\{P_{h}\}$. 
This is a typical feature of the family of batch algorithms, to which Algorithm \ref{alg:alg_lsvi} belongs. 
Then, in Line \ref{line:lsvi_Q_update},
the algorithm uses the intermediate value $w_{h}$ to compute the
estimate $Q_h$, by adding/subtracting bonus $b_{h}$ and thresholding
the sum/difference at $e^{\beta(H-h+1)}$, depending on the sign of
$\beta$. It is not hard to see that the logarithmic-exponential transformation
in Line \ref{line:lsvi_Q_update} conforms and adapts to the non-linearity
in Bellman equations \eqref{eq:bellman_primal} and \eqref{eq:bellman_optimal_primal}.
In addition, the thresholding operator ensures that the estimated action-value
function $Q_{h}$ of step $h$ stays in the range $[0,H-h+1]$ and
so does the estimated value function $V_{h}$ in Line \ref{line:LSVI_value}.
This is to enforce the estimates $Q_{h}$ and $V_{h}$ to be on the
same scale as the optimal $Q_{h}^{*}$ and $V_{h}^{*}$.

Besides the logarithmic-exponential transformation, another distinctive
feature of Algorithm \ref{alg:alg_lsvi} is the way the bonus term
$b_{h} >0$ is incorporated in Line \ref{line:lsvi_Q_update}. At first
sight, it might appear counter-intuitive to \emph{subtract} $b_{h}$ from
$w_{h}$ when $\beta<0$. We demonstrate next that subtracting bonus when $\beta<0$
in fact implements the idea of OFU in a risk-sensitive fashion. 

\begin{algorithm}[t]
	\begin{algorithmic}[1]
		
		\Require number of episodes $K\in\mathbb{Z}_{>0}$, confidence level
		$\delta\in(0,1]$, and risk parameter $\beta\ne0$
		
		\State $Q_{h}(s,a)\leftarrow H-h+1$ and $N_h(s,a) \leftarrow 0$ for all $(h,s,a)\in[H]\times\cS\times\cA$
		
		\State $Q_{H+1}(s,a)\leftarrow0$ for all $(s,a)\in\cS\times\cA$
		
		\State Initialize datasets $\{\cD_{h}\}$ as empty
		
		
		\For{episode $k=1,\ldots,K$}
		
		\State $V_{H+1}(s)\leftarrow0$ for each $s\in\cS$
		
		\For{step $h=H,\ldots,1$}\Comment{\emph{value estimation}}\label{line:LSVI_estim_value_begin}
		
		\State Update $\w_{h}$ via Equation \eqref{eq:LSVI_LS_step} \label{line:LSVI_interm_value}
		
		\For{ $(s,a)\in\cS\times\cA$ such that $N_{h}(s,a)\ge1$}
		
		
		\State \label{line:lsvi_bonus_def}$b_{h}(s,a)\leftarrow c_{\gamma}\left|e^{\beta H}-1\right|\sqrt{\frac{S\log(2SAT/\delta)}{N_{h}(s,a)}}$	for some universal constant $c_{\gamma}>0$
		
		
		\State$Q_{h}(s,a)\leftarrow\begin{cases}
		\frac{1}{\beta}\log\left[\min\{e^{\beta(H-h+1)},\w_{h}(s,a)+b_{h}(s,a)\}\right], & \text{if }\beta>0;\\
		\frac{1}{\beta}\log\left[\max\{e^{\beta(H-h+1)},\w_{h}(s,a)-b_{h}(s,a)\}\right], & \text{if }\beta<0
		\end{cases}$\label{line:lsvi_Q_update}
		
		\State $V_{h}(s)\leftarrow\max_{a'\in\cA}Q_{h}(s,a')$ \label{line:LSVI_value}
		
		\EndFor 
		
		\EndFor \label{line:LSVI_estim_value_end}
		
		\For{ step $h=1,\ldots,H$}\Comment{\emph{policy execution}}\label{line:LSVI_exec_policy_begin}
		
		
		\State Take action $a_{h}\leftarrow\argmax_{a\in\cA}Q_{h}(s_{h},a)$
			and observe $r_{h}(s_{h},a_{h})$ and $s_{h+1}$
		
		\State $N_{h}(s_{h},a_{h})\leftarrow N_{h}(s_{h},a_{h})+1$
		
		\State Insert $(s_{h},a_{h},s_{h+1})$ into $\cD_{h}$

		
		
		
		\EndFor \label{line:LSVI_exec_policy_end}
		
		\EndFor
		
	\end{algorithmic}
	
	\caption{RSVI
	\label{alg:alg_lsvi}}
\end{algorithm}

\paragraph{Risk-Sensitive Upper Confidence Bound.}

For the purpose of illustration, let us consider a ``promising''
state $s^{+}\in\cS$ at step $h$ that allows us to transition to
states \{$s'\}$ in the next step with high values $\{V_{h+1}(s')\}$
regardless of actions taken. This means that the intermediate value
$w_{h}(s^{+},\cdot)\propto\sum_{s'}e^{\beta\cdot V_{h+1}(s')}$ tends
to be \emph{small, }given that $\beta<0$ and $\{V_{h+1}(s')\}$ are
large. By subtracting a positive $b_{h}$ from $w_{h}$, we obtain
an even smaller quantity $w_{h}(s^{+},\cdot)-b_{h}(s^{+},\cdot)$. We can then
deduce that $Q_{h}(s^{+},\cdot)\approx\frac{1}{\beta}\log[w_{h}(s^{+},\cdot)-b_{h}(s^{+},\cdot)]$
is \emph{larger} compared to $\frac{1}{\beta}\log[w_{h}(s^{+},\cdot)]$
which does not incorporate bonus, since the logarithmic function is
monotonic and again $\beta<0$ (we ignore thresholding for the moment).
Therefore, subtracting bonus serves as a UCB for $\beta<0$ . Since
the exact form of the UCB depends on both the magnitude and sign of
$\beta$ (as shown in Lines \ref{line:lsvi_bonus_def} and \ref{line:lsvi_Q_update}),
we name it Risk-Sensitive Upper Confidence Bound (RS-UCB) and this
results in what we call Risk-Sensitive Optimism in the Face of Uncertainty
(RS-OFU). 

\subsection{Risk-Sensitive Q-learning}

Although Algorithm \ref{alg:alg_lsvi} is model-free, it requires
storage of historical data $\{\cD_{h}\}$ and computation over them
(Line \ref{line:LSVI_interm_value}). A more efficient class of algorithms
is Q-learning algorithms, which update Q values in an online fashion
as each state-action pair is encountered. We therefore propose 
Risk-Sensitive Q-learning (RSQ)
and  formally describe it in Algorithm \ref{alg:Q_learn}.
\begin{algorithm}[t]
	\begin{algorithmic}[1]
		
		\Require number of episodes $K\in\mathbb{Z}_{>0}$, confidence level
		$\delta\in(0,1]$, learning rates $\{\alpha_{t}$\} and risk parameter
		$\beta\ne0$
		
		\State $Q_{h}(s,a), V_{h}(s,a) \leftarrow H-h+1$ and $N_{h}(s,a)\leftarrow0$
		for all $(h,s,a)\in[H]\times\cS\times\cA$
		
		
		\State $Q_{H+1}(s,a), V_{H+1}(s,a) \leftarrow0$ for all $(s,a)\in\cS\times \cA$
		
		\For{episode $k=1,\ldots,K$}
		
		\State Receive the initial state $s_{1}$
		
		\For{step $h=1,\ldots,H$} \label{line:qlearn_estim_value_begin}
		
		\State Take action $a_{h}\leftarrow\argmax_{a'\in\cA}Q_{h}(s_{h},a')$,
		and observe $r_{h}(s_{h},a_{h})$ and $s_{h+1}$\label{line:qlearn_exec_policy}
		
		\State $t=N_{h}(s_{h},a_{h})\leftarrow N_{h}(s_{h},a_{h})+1$
		
		\State $b_{t}\leftarrow c\left|e^{\beta H}-1\right|\sqrt{\frac{H\log(SAT/\delta)}{t}}$
		for some sufficiently large universal constant $c>0$ \label{line:qlearn_bonus_def}
		
		\State $w_{h}(s_{h},a_{h})\leftarrow(1-\alpha_{t})e^{\beta\cdot Q_{h}(s_{h},a_{h})}+\alpha_{t}e^{\beta[r_{h}(s_{h},a_{h})+V_{h+1}(s_{h+1})]}$
		\label{line:qlearn_interm_value}
		
		\State $Q_{h}(s_{h},a_{h})\leftarrow\begin{cases}
		\frac{1}{\beta}\log\left[\min\{e^{\beta(H-h+1)},\w_{h}(s_{h},a_{h})+\alpha_{t}b_{t}\}\right], & \text{if }\beta>0;\\
		\frac{1}{\beta}\log\left[\max\{e^{\beta(H-h+1)},\w_{h}(s_{h},a_{h})-\alpha_{t}b_{t}\}\right], & \text{if }\beta<0
		\end{cases}$\label{line:qlearn_Q_update}
		
		\State $V_{h}(s_{h})\leftarrow\max_{a'\in\cA}Q_{h}(s_{h},a')$\label{line:qlearn_value}
		
		\EndFor \label{line:qlearn_estim_value_end}
		
		\EndFor
		
	\end{algorithmic}
	
	\caption{RSQ\label{alg:Q_learn}}
\end{algorithm}

\paragraph{Mechanism of Algorithm \ref{alg:Q_learn}.}

Algorithm \ref{alg:Q_learn} is based on Q-learning with UCB studied in the
work of \citet{jin2018q} and we use the same learning rates therein
\begin{equation}
\alpha_{t}\coloneqq\frac{H+1}{H+t}\label{eq:learn_rate}
\end{equation}
for every integer $t\ge1$.  Similar to Algorithm \ref{alg:alg_lsvi},
Algorithm \ref{alg:Q_learn} consists of the policy execution step
(Line \ref{line:qlearn_exec_policy}) and value estimation step (Lines
\ref{line:qlearn_interm_value}--\ref{line:qlearn_value}). Line
\ref{line:qlearn_interm_value} updates the intermediate value $w_{h}$
in an online fashion, in constrast with the batch update in Line \ref{line:LSVI_interm_value}
of Algorithm \ref{alg:alg_lsvi}, and Algorithm \ref{alg:Q_learn} can thus be seen as an online algorithm. Line \ref{line:qlearn_Q_update} then applies the same
logarithmic-exponential transform to the intermediate value and bonu 
as in  Algorithm \ref{alg:alg_lsvi}. Note
the similar way we use the bonus term $b_{t}$ in estimating $Q$-values
in Line \ref{line:qlearn_Q_update} of Algorithm \ref{alg:Q_learn} as in Line \ref{line:lsvi_Q_update}
of Algorithm \ref{alg:alg_lsvi}. Algorithm \ref{alg:Q_learn} therefore
also implements RS-UCB and follows the principle of RS-OFU.

\paragraph{Comparisons of Algorithms \ref{alg:alg_lsvi} and \ref{alg:Q_learn}.}

It is interesting to compare the bonuses used in Algorithms \ref{alg:alg_lsvi}
and \ref{alg:Q_learn}. The bonuses in both algorithms depend on the
risk parameter $\beta$ through a common factor $\left|e^{\beta H}-1\right|$.
A careful analysis (see our proofs in appendices) on the bonuses and the value estimation
steps reveals that the effective bonuses added to the estimated value function
is proportional to  $\frac{e^{\left|\beta\right|H}-1}{\left|\beta\right|}$.
This means that the more risk-seeking/averse an agent is (or the larger
$\left|\beta\right|$ is), the larger bonus it needs to compensate
for its uncertainty over the environment. Such risk sensitivity of
the bonus is also reflected in the regret bounds; see Theorems \ref{thm:regret_LSVI}
and \ref{thm:regret_Q_learn} below. 
Also, it is not hard to see that both algorithms have polynomial
time and space complexities in $S$, $A$, $K$ and $H$. Moreover, thanks to its online update procedure,
Algorithm \ref{alg:Q_learn} is more efficient than Algorithms \ref{alg:alg_lsvi}
in both time and space complexities, since it does not require storing
historical data (in particular, $\{\cD_{h}\}$ of Algorithm \ref{alg:alg_lsvi})
nor computing statistics based on them for value estimation.

\section{Main results\label{sec:main}}

In this section, we first present regret bounds for Algorithms \ref{alg:alg_lsvi}
and \ref{alg:Q_learn},
and then we complement the
results with a lower bound on regret that any algorithm has to incur.


\subsection{Regret upper bounds}

The following theorem gives an upper bound for regret incurred by Algorithm \ref{alg:alg_lsvi}.
Let  $ T \coloneqq KH $ be the total number of timesteps for which an algorithm is run, and recall the function $\lambda(u) \coloneqq (e^{3u}-1)/u$.

\begin{thm}
	\label{thm:regret_LSVI} For any $\delta\in(0,1]$, with probability
	at least $1-\delta$, the regret of Algorithm \ref{alg:alg_lsvi}
	is bounded by 
	\[
	\reg(K)\lesssim \lambda(|\beta| H^2) \cdot\sqrt{H^{3}S^{2}AT\log^{2}(2SAT/\delta)}.
	\]
\end{thm}
The proof is given in Appendix \ref{sec:proof_regret_lsvi}.
We see that the result of Theorem \ref{thm:regret_LSVI} adapts to
both risk-seeking ($\beta>0$) and risk-averse ($\beta<0$) settings
through a common factor of
$\lambda(|\beta| H^2)$.

As $\beta\to0$, the setting of risk-sensitive RL tends
to that of standard and risk-neutral RL, and we have an immediate
corollary to Theorem \ref{thm:regret_LSVI} as a precise characterization.
\begin{cor}
	\label{cor:regret_lsvi_beta_0} Under the setting of Theorem \ref{thm:regret_LSVI}
	and when $\beta\to0$, with probability at least $1-\delta$, the
	regret of Algorithm \ref{alg:alg_lsvi} is bounded by 
	\begin{align*}
	\reg(K) & \lesssim\sqrt{H^{3}S^{2}AT\log^{2}(2SAT/\delta)}.
	\end{align*}
\end{cor}
\begin{proof}
	The result follows from Theorem \ref{thm:regret_LSVI} and the fact
	that $\lim_{\beta\to 0} \lambda(|\beta|H^2) = 3$.
\end{proof}
The result in Corollary \ref{cor:regret_lsvi_beta_0}  recovers the
regret bound of \citet[Theorem 2]{bai2020provable} under the standard
RL setting and is nearly optimal compared to the minimax rates presented
in \citet[Theorems 1 and 2]{azar2017minimax}. 
Corollary \ref{cor:regret_lsvi_beta_0}
also reveals that Theorem \ref{thm:regret_LSVI} 
interpolates between the risk-sensitive and risk-neutral settings.

Next, we give a regret upper bound for Algorithm \ref{alg:Q_learn}
in the following theorem.
\begin{thm}
	\label{thm:regret_Q_learn}For any $\delta\in(0,1]$, with probability
	at least $1-\delta$ and when $T$ is sufficiently large, the regret
	of Algorithm \ref{alg:Q_learn} is bounded by 
	\begin{align*}
	\reg(K) & \lesssim\ \lambda(|\beta| H^2) \cdot\sqrt{H^{4} SAT\log(SAT/\delta)}.
	\end{align*}
\end{thm}
The proof is given in Appendix \ref{sec:proof_regret_Q_learn}. 
Similarly to Theorem \ref{thm:regret_LSVI}, Theorem \ref{thm:regret_Q_learn} also covers both risk-seeking and risk-averse settings via the same factor $\lambda(|\beta| H^2)$,
which gives the risk-neutral bound when $\beta\to0$ as shown in the following.

\begin{cor}
	\label{cor:regret_Q_learn_beta_0}Under the setting of Theorem \ref{thm:regret_Q_learn}
	and when $\beta\to0$, with probability at least $1-\delta$, the
	regret of Algorithm \ref{alg:Q_learn} is bounded by 
	\begin{align*}
	\reg(K) & \lesssim\sqrt{H^{4}SAT\log(SAT/\delta)}.
	\end{align*}
\end{cor}
The proof follows the same reasoning as in that of Corollary \ref{cor:regret_lsvi_beta_0}.
According to Corollary \ref{cor:regret_Q_learn_beta_0}, the regret
upper bound for Algorithm \ref{alg:Q_learn} matches the nearly optimal result in
\citet[Theorem 2]{jin2018q} under the risk-neutral setting. 
As such, Theorems \ref{thm:regret_LSVI} and \ref{thm:regret_Q_learn} strictly generalizes the existing nearly optimal regret bounds (up to polynomial factors).


The crux of the proofs of both Theorems \ref{thm:regret_LSVI} and \ref{thm:regret_Q_learn}
lies in a local linearization argument for the non-linear Bellman equations
and non-linear updates of the algorithms, in which action-value and value functions
are related by a logarithmic-exponential transformation. Although
logarithmic and exponential functions are not Lipschitz globally,
we show that they are locally Lipschitz in the domain of our interest,
and their combined local Lipschitz factors turn out to be the exponential
factors in the theorems. Once the Bellman equations and algorithm estimates are linearized,
we can apply standard techniques in RL to obtain the final regret. 
It is noteworthy
that, as suggested by \citet{jin2018q}, the regret bounds in Theorems
\ref{thm:regret_LSVI} and \ref{thm:regret_Q_learn} can automatically
be translated into sample complexity bounds in the probably approximately
correct (PAC) setting, which did not previously exist even given access
to a simulator.

In the risk-sensitive setting where $ \beta  $ is bounded away from $ 0 $, our regret bounds of Theorems
\ref{thm:regret_LSVI} and \ref{thm:regret_Q_learn} depend exponentially in the horizon length $ H $ and the risk sensitivity $ |\beta| $. In what follows, we argue that such exponential dependence is unavoidable. 


\subsection{Regret lower bound}

We now  present a fundamental lower bound on the regret, which complements the upper bounds in Theorems 
\ref{thm:regret_LSVI} and \ref{thm:regret_Q_learn}.
\begin{thm}
	\label{thm:regret_lower_bound}\emph{ }For sufficiently large $K$
	and $H$, the regret of any algorithm obeys 
	\begin{align*}
	\E\left[\reg(K)\right] & \gtrsim\frac{e^{\left|\beta\right|H/2}-1}{\left|\beta\right|}\sqrt{T\log T}.
	\end{align*}
\end{thm}

The proof is given in Appendix \ref{sec:proof_regret_lower_bound}. In the proof,
we construct a bandit model that can be seen as a special case of our episodic fixed-horizon MDP problem, and 
then we show that any bandit algorithm has to incur an expected regret, in terms of the logarithmic-exponential objective, that grows as predicted in Theorem~\ref{thm:regret_lower_bound}.

Theorem \ref{thm:regret_lower_bound} shows that the exponential dependence on the $ |\beta| $ and $ H $ in Theorems 
\ref{thm:regret_LSVI} and \ref{thm:regret_Q_learn} is essentially indispensable. In addition, it features a sub-linear dependence on $ T $ through the $\tilde{O}(\sqrt{T})$ factor. In view of Theorem \ref{thm:regret_lower_bound}, therefore, both 
Theorems \ref{thm:regret_LSVI} and \ref{thm:regret_Q_learn} are nearly optimal in their dependence on $ \beta $, $H$ and $T$.
One should contrast Theorem~\ref{thm:regret_lower_bound} with Lemma~\ref{lem:simple_bound}, which shows that the worst-case regret is linear in $ H $ and $ T $. Such a linear regret can be attained by any trivial algorithm that does not learn at all. In sharp contrast, in order to achieve the optimal $ \sqrt{T}  $ scaling (which by standard arguments implies a finite sample-complexity bound), an algorithm must incur a regret that is exponential in $ H $. Therefore, our results show a (perhaps surprising) tradeoff between risk sensitivity and sample efficiency.

\section*{Acknowledgement}

Y. Fei and Y. Chen were supported in part by National Science Foundation Grant CCF-1704828.

\nocite{}
\bibliographystyle{plain}
\bibliography{references}

\newpage

\appendix
\appendixpage

\section{Preliminaries}

We set some notations and shorthands before the proofs. For both Algorithms
\ref{alg:alg_lsvi} and \ref{alg:Q_learn}, we let $s_h^k$, $a_h^k$, $w_{h}^{k}$, $Q_{h}^{k}$
and $V_{h}^{k}$ denote the values of $s_h$, $a_h$, $w_{h}$, $Q_{h}$ and $V_{h}$
in episode $k$, and we denote by $N_{h}^{k}$ the value of $N_{h}$
at the end of episode $k-1$. For Algorithm \ref{alg:alg_lsvi}, we
let $\cD_{h}^{k}$ be the value of $\cD_{h}$ at the end of episode
$k-1$. Next, we introduce a simple yet powerful result.
\begin{fact}
	Consider $x,y,b\in\real$ such that $x\ge y$.
	\begin{enumerate}[label={(\alph*)},ref={\thefact(\alph*)}]
		\item \label{fact:log_lip}if $y\ge g$ for some $g>0$, then $\log(x)-\log(y)\le\frac{1}{g}(x-y)$;
		\item \label{fact:exp_lip}Assume further that $y\ge0$. If $b\ge0$ and
		$x\le u$ for some $u>0$, then $e^{bx}-e^{by}\le be^{bu}(x-y)$;
		if $b<0$, then $e^{by}-e^{bx}\le(-b)(x-y)$. 
	\end{enumerate}
\end{fact}
\begin{proof}
	The results follow from Lipschitz continuity of the functions $x\mapsto\log(x)$
	and $x\mapsto e^{bx}$.
\end{proof}
We record a simple fact about exponential factors.
\begin{fact}
	\label{fact:exp_factor}Define $\lambda_{0}\coloneqq\frac{e^{\left|\beta\right|H}-1}{\left|\beta\right|}$
	and $\lambda_{2}\coloneqq e^{\left|\beta\right|(H^{2}+H)}$. Then
	we have $\lambda_{0}\lambda_{2}H\le\frac{e^{3\left|\beta\right|H^{2}}-1}{\left|\beta\right|}$.
\end{fact}

\section{Proof warmup for Theorem \ref{thm:regret_LSVI}}

First, we set some notations and definitions. Define $d\coloneqq SA$,
$\iota\coloneqq\log(2dT/\delta)$ for a given $\delta\in(0,1]$, and
$\IdMat$ to be the $d\times d$ identity matrix. To streamline some
parts of the proof, we define $\fmap(s,a)$ to be a vector in $\real^{d}$
whose $(s,a)$-th entry is equal to one and other entries equal to
zero (so $\phi(s,a)$ is a canonical basis of $\real^{SA}$). Also let $\Lambda_{h}^{k}$ be a diagonal matrix in $\real^{d\times d}$
with each $(s,a)$-th diagonal entry equal to $\max\{N_{h}^{k-1}(s,a),1\}$.
It can be seen that $\Lambda_{h}^{k}$ is positive definite. We adopt
the shorthands $\fmap_{h}^{\tau}\coloneqq\fmap(s_{h}^{\tau},a_{h}^{\tau})$
and  $r_{h}^{\tau}\coloneqq r_{h}(s_{h}^{\tau},a_{h}^{\tau})$ for
$(\tau,h)\in[K]\times[H]$.

From now on, we fix a tuple $(k,h)\in[K]\times[H]$ and then fix $(s,a)\in\cS\times\cA$
such that $N_{h}^{k-1}(s,a)\ge1$. We also fix a policy $\pi$. We
set 
\begin{equation}
\w_{h}^{\pi}=e^{\beta\cdot Q_{h}^{\pi}(\cdot,\cdot)}.\label{eq:w^pi_h_def}
\end{equation}
It can be verified that by the definition of $\fmap(s,a)$, we have
\begin{align}
Q_{h}^{\pi}(s,a) & =\frac{1}{\beta}\log\left(e^{\beta\cdot Q_{h}^{\pi}(s,a)}\right)\nonumber \\
& =\frac{1}{\beta}\log\left(\left\langle \fmap(s,a),e^{\beta\cdot Q_{h}^{\pi}(\cdot,\cdot)}\right\rangle \right)\nonumber \\
& =\frac{1}{\beta}\log\left(\left\langle \fmap(s,a),\w_{h}^{\pi}\right\rangle \right),\label{eq:lsvi_Q^pi_equiv}
\end{align}
as well as 
\begin{align}
\w_{h}^{\pi}(s,a) & =e^{\beta\cdot Q_{h}^{\pi}(s,a)}=\left\langle \fmap(s,a),(\Lambda_{h}^{k})^{-1}\sum_{\tau\in[k-1]}\fmap_{h}^{\tau}\left[e^{\beta\cdot Q_{h}^{\pi}(s_{h}^{\tau},a_{h}^{\tau})}\right]\right\rangle ,\label{eq:lsvi_w^pi_h_equiv}
\end{align}
where the last step follows from the definition of $\Lambda_{h}^{k}$.

Let us define 
\begin{align*}
q_{1}^{+} & \coloneqq\begin{cases}
\left\langle \fmap(s,a),\w_{h}^{k}\right\rangle +b_{h}^{k}(s,a), & \text{if }\beta>0,\\
\left\langle \fmap(s,a),\w_{h}^{k}\right\rangle -b_{h}^{k}(s,a), & \text{if }\beta<0,
\end{cases}\\
q_{1} & \coloneqq\begin{cases}
\min\{e^{\beta(H-h+1)},q_{1}^{+}\}, & \text{if }\beta>0,\\
\max\{e^{\beta(H-h+1)},q_{1}^{+}\}, & \text{if }\beta<0.
\end{cases}
\end{align*}
By the definition of $\Lambda_{h}^{k}$ and $\fmap_{h}^{k}$, observe
that 
\begin{equation}
w_{h}^{k}(s,a)=\left\langle \fmap(s,a),\w_{h}^{k}\right\rangle =\left\langle \fmap(s,a),(\Lambda_{h}^{k})^{-1}\sum_{\tau\in[k-1]}\fmap_{h}^{\tau}\left[e^{\beta[r_{h}^{\tau}+V_{h+1}^{k}(s_{h+1}^{\tau})]}\right]\right\rangle .\label{eq:lsvi_weights_equiv}
\end{equation}
Define 
\begin{equation}
G_{0}\coloneqq(Q_{h}^{k}-Q_{h}^{\pi})(s,a)=\frac{1}{\beta}\log\left\{ q_{1}\right\} -\frac{1}{\beta}\log\left\{ \left\langle \fmap(s,a),\w_{h}^{\pi}\right\rangle \right\} ,\label{eq:G0_def}
\end{equation}
and our goal is to derive lower and upper bounds for $G_{0}$. From
Equation \eqref{eq:G0_def}, we have 
\begin{align*}
G_{0} & =\frac{1}{\beta}\log\left\{ q_{1}\right\} -\frac{1}{\beta}\log\left\{ \left\langle \fmap(s,a),(\Lambda_{h}^{k})^{-1}\sum_{\tau\in[k-1]}\fmap_{h}^{\tau}\left[e^{\beta\cdot Q_{h}^{\pi}(s_{h}^{\tau},a_{h}^{\tau})}\right]\right\rangle \right\} \\
& =\frac{1}{\beta}\log\left\{ q_{1}\right\} -\frac{1}{\beta}\log\left\{ \left\langle \fmap(s,a),(\Lambda_{h}^{k})^{-1}\sum_{\tau\in[k-1]}\fmap_{h}^{\tau}\left[\E_{s'\sim P_{h}(\cdot\,|\,s_{h}^{\tau},a_{h}^{\tau})}e^{\beta[r_{h}^{\tau}+V_{h+1}^{\pi}(s')]}\right]\right\rangle \right\} \\
& \eqqcolon\frac{1}{\beta}\log\{q_{1}\}-\frac{1}{\beta}\log\{q_{3}\}.
\end{align*}
The first step above holds by Equation \eqref{eq:lsvi_w^pi_h_equiv},
and the second step follows from Equation \eqref{eq:bellman_primal}.
In order to control $G_{0}$, we define an intermediate quantity
\[
q_{2}\coloneqq\left\langle \fmap(s,a),(\Lambda_{h}^{k})^{-1}\sum_{\tau\in[k-1]}\fmap_{h}^{\tau}\left[\E_{s'\sim P_{h}(\cdot\,|\,s_{h}^{\tau},a_{h}^{\tau})}e^{\beta[r_{h}^{\tau}+V_{h+1}^{k}(s')]}\right]\right\rangle ;
\]
in words, $q_{2}$ replaces the quantity $V_{h+1}^{\pi}$ in $q_{3}$
by $V_{h+1}^{k}$. It can be seen that  
\begin{equation}
G_{0}=G_{1}+G_{2},\label{eq:G0_decomp}
\end{equation}
where 
\begin{equation}
\begin{aligned}G_{1} & \coloneqq\frac{1}{\beta}\log\{q_{1}\}-\frac{1}{\beta}\log\{q_{2}\},\\
G_{2} & \coloneqq\frac{1}{\beta}\log\{q_{2}\}-\frac{1}{\beta}\log\{q_{3}\}.
\end{aligned}
\label{eq:G1_G2_def}
\end{equation}
Note that $G_{0}$, $G_{1}$ and $G_{2}$ are all well-defined, according
to the following result.
\begin{lem}
	\label{lem:q_lower_bound}We have $q_{i}\in[\min\{1,e^{\beta(H-h+1)}\},\max\{1,e^{\beta(H-h+1)}\}]$
	for $i\in[3]$. 
\end{lem}
\begin{proof}
	We prove the result by focusing on $q_{1}$.  By the definitions
	of $\Lambda_{h}^{k}$ and $\fmap$, the $(s,a)$-th entry of the vector
	$(\Lambda_{h}^{k})^{-1}\sum_{\tau\in[k-1]}\fmap_{h}^{\tau}\cdot u_{h}^{\tau}$
	equals $\frac{1}{N_{h}^{k-1}(s,a)}\sum_{\tau\in[k-1]}u_{h}^{\tau}\cdot\indic\{(s_{h}^{\tau},a_{h}^{\tau})=(s,a)\}$
	for any sequence $\{u_{h}^{\tau}\}_{\tau\in[k-1]}$. Then, the result
	follows from the fact that $e^{\beta[r_{h}^{\tau}+V_{h+1}^{k}(s')]}\in[\min\{1,e^{\beta(H-h)}\},\max\{1,e^{\beta(H-h)}\}]$
	for $(\tau,s')\in[K]\times\cS$ and the definition of $q_{1}$. 
\end{proof}
Therefore, we have the following equivalent form of Equation \eqref{eq:G0_def}:
\begin{equation}
(Q_{h}^{k}-Q_{h}^{\pi})(s,a)=G_{1}+G_{2}.\label{eq:lsvi_Q_diff_equiv}
\end{equation}
Thanks to the identity
\eqref{eq:lsvi_Q_diff_equiv}, our goal is now to control $G_{1}$ and $G_{2}$,  which is done in the following lemma.
\begin{lem}
	\label{lem:G1_G2_bound}For all $(k,h,s,a)\in[K]\times[H]\times\cS\times\cA$
	that satisifies $N_{h}^{k-1}(s,a)\ge1$, there exist universal constants
	$c_{1},c_{\gamma}>0$ (where $c_{\gamma}$ is used in Line \ref{line:lsvi_bonus_def}
	of Algorithm \ref{alg:alg_lsvi}) such that 
	\[
	0\le G_{1}\le c_{1}\cdot\frac{e^{\left|\beta\right|H}-1}{\left|\beta\right|}\cdot d\sqrt{\iota}\sqrt{\fmap(s,a)^{\top}(\Lambda_{h}^{k})^{-1}\fmap(s,a)}
	\]
	with probability at least $1-\delta/2$. Furthermore, if $V_{h+1}^{k}(s')\ge V_{h+1}^{\pi}(s')$
	for all $s'\in\cS$, then we have 
	\[
	0\le G_{2}\le e^{\left|\beta\right|H}\cdot\E_{s'\sim P_{h}(\cdot\,|\,s,a)}[V_{h+1}^{k}(s')-V_{h+1}^{\pi}(s')].
	\]
\end{lem}
\begin{proof}
	\textbf{Case $\beta>0$. }To control $G_{1}$, we note that $N_{h}^{k-1}(s,a)=\fmap(s,a)^{\top}(\Lambda_{h}^{k})^{-1}\fmap(s,a)$
	and by Equation \eqref{eq:lsvi_weights_equiv} we can compute 
	\begin{align*}
	& \quad\left|q_{1}^{+}-q_{2}-b_{h}^{k}(s,a)\right|\\
	& =\left|\left\langle \fmap(s,a),(\Lambda_{h}^{k})^{-1}\sum_{\tau\in[k-1]}\fmap_{h}^{\tau}\left[e^{\beta[r_{h}^{\tau}+V_{h+1}^{k}(s_{h+1}^{\tau})]}-\E_{s'\sim P_{h}(\cdot\,|\,s_{h}^{\tau},a_{h}^{\tau})}e^{\beta[r_{h}^{\tau}+V_{h+1}^{k}(s')]}\right]\right\rangle \right|\\
	& =\left|\frac{1}{N_{h}^{k-1}(s,a)}\sum_{(s,a,s^{+})\in\cD_{h}^{k-1}}e^{\beta[r_{h}(s,a)+V_{h+1}^{k}(s^{+})]}-\E_{s'\sim P_{h}(\cdot\,|\,s,a)}e^{\beta[r_{h}(s,a)+V_{h+1}^{k}(s')]}\right|\\
	& \le\frac{1}{N_{h}^{k-1}(s,a)}\sum_{(s,a,s^{+})\in\cD_{h}^{k-1}}\left|e^{\beta[r_{h}(s,a)+V_{h+1}^{k}(s^{+})]}-\E_{s'\sim P_{h}(\cdot\,|\,s,a)}e^{\beta[r_{h}(s,a)+V_{h+1}^{k}(s')]}\right|\\
	& \le\frac{1}{N_{h}^{k-1}(s,a)}\sum_{t\in\left[N_{h}^{k-1}(s,a)\right]}c'\left|e^{\beta H}-1\right|\sqrt{\frac{S\iota}{t}}\\
	& \le\frac{1}{N_{h}^{k-1}(s,a)}\int_{t\in\left[0,N_{h}^{k-1}(s,a)\right]}c'\left|e^{\beta H}-1\right|\sqrt{\frac{S\iota}{t}}\mathrm{d}t\\
	& =\frac{1}{N_{h}^{k-1}(s,a)}\cdot c\left|e^{\beta H}-1\right|\sqrt{S\iota\cdot N_{h}^{k-1}(s,a)}\\
	& =c\left|e^{\beta H}-1\right|\sqrt{S\iota}\cdot\sqrt{\fmap(s,a)^{\top}(\Lambda_{h}^{k})^{-1}\fmap(s,a)},
	\end{align*}
	where the fourth step holds by Lemma \ref{lem:LSVI_concen}, and the
	last step holds by the definition of $\Lambda_{h}^{k}$; in the above,
	$c'>0$ is a universal constant and $c=2c'$.  If we choose $c_{\gamma}=c$
	in the definition of $b_{h}^{k}(s,a)$ in Line \ref{line:lsvi_bonus_def}
	of Algorithm \ref{alg:alg_lsvi}, we have 
	\[
	0\le q_{1}^{+}-q_{2}\le2c\cdot\left|e^{\beta H}-1\right|\sqrt{S\iota}\cdot\sqrt{\fmap(s,a)^{\top}(\Lambda_{h}^{k})^{-1}\fmap(s,a)}.
	\]
	Therefore, we have $q_{1}\ge q_{2}$, and thus $G_{1}\ge0$, by the
	first inequality above, the definition of $q_{1}$ and Lemma \ref{lem:q_lower_bound}
	(in particular, $q_{2}\le e^{\beta(H-h+1)}$). By Lemma \ref{lem:q_lower_bound}
	and Fact \ref{fact:log_lip} (with $g=1$, $x=q_{1}$ and $y=q_{2}$),
	we have 
	\[
	G_{1}\le\frac{1}{\beta}(q_{1}-q_{2})\le\frac{1}{\beta}(q_{1}^{+}-q_{2}),
	\]
	which together with the second inequality displayed above implies
	the desired upper bound on $G_{1}$.
	
	Now we control the term $G_{2}$. For $\beta>0$, it is not hard to
	see that the assumption $V_{h+1}^{k}(s')\ge V_{h+1}^{\pi}(s')$ for
	all $s'\in\cS$ implies that $q_{2}\ge q_{3}$ and therefore $G_{2}\ge0$.
	We also have 
	\begin{align*}
	G_{2} & \le\frac{1}{\beta}(q_{2}-q_{3})\\
	& \le e^{\beta H}\left\langle \fmap(s,a),(\Lambda_{h}^{k})^{-1}\sum_{\tau\in[k-1]}\fmap_{h}^{\tau}\left[\E_{s'\sim P_{h}(\cdot\,|\,s_{h}^{\tau},a_{h}^{\tau})}[V_{h+1}^{k}(s')-V_{h+1}^{\pi}(s')]\right]\right\rangle \\
	& =e^{\left|\beta\right|H}\E_{s'\sim P_{h}(\cdot\,|\,s,a)}[V_{h+1}^{k}(s')-V_{h+1}^{\pi}(s')],
	\end{align*}
	where the first step holds by Fact \ref{fact:log_lip} (with $g=1$,
	$x=q_{2}$, and $y=q_{3}$) and the fact that $q_{2}\ge q_{3}\ge1$
	(with the last inequality suggested by Lemma \ref{lem:q_lower_bound}),
	and the second step holds by Fact \ref{fact:exp_lip} (with $b=\beta$,
	$x=r_{h}^{\tau}+V_{h+1}^{k}(s)$, and $y=r_{h}^{\tau}+V_{h+1}^{\pi}(s)$)
	and $H\ge r_{h}^{\tau}+V_{h+1}^{k}(s)\ge r_{h}^{\tau}+V_{h+1}^{\pi}(s)\ge0$. 
	
	\textbf{Case $\beta<0$. }Similar to the case of $\beta>0$, we have
	\begin{align*}
	& \quad\left|q_{1}^{+}-q_{2}+b_{h}^{k}(s,a)\right|\\
	& \le c\cdot\left|e^{\beta H}-1\right|\sqrt{S\iota}\cdot\sqrt{\fmap(s,a)^{\top}(\Lambda_{h}^{k})^{-1}\fmap(s,a)}.
	\end{align*}
	If we choose $c_{\gamma}=c$ in the definition of $b_{h}^{k}(s,a)$
	in Line \ref{line:lsvi_bonus_def} of Algorithm \ref{alg:alg_lsvi},
	the above equation implies 
	\[
	0\le q_{2}-q_{1}^{+}\le2c\cdot\left|e^{\beta H}-1\right|\sqrt{S\iota}\cdot\sqrt{\fmap(s,a)^{\top}(\Lambda_{h}^{k})^{-1}\fmap(s,a)}.
	\]
	Therefore, we have $q_{1}\le q_{2}$, and thus $G_{1}\ge0$, by the
	first inequality displayed above, the definition of $q_{1}$ and Lemma
	\ref{lem:q_lower_bound} (in particular, $q_{2}\ge e^{\beta(H-h+1)}$).
	By Lemma \ref{lem:q_lower_bound} and Fact \ref{fact:log_lip} (with
	$g=e^{\beta H}$, $x=q_{2}$ and $y=q_{1})$, we further have 
	\begin{align*}
	G_{1} & =\frac{1}{(-\beta)}\left(\log\{q_{2}\}-\log\{q_{1}\}\right)\\
	& \le\frac{e^{-\beta H}}{\left|\beta\right|}(q_{2}-q_{1})\\
	& \le\frac{e^{-\beta H}}{\left|\beta\right|}(q_{2}-q_{1}^{+}),
	\end{align*}
	which together with the second inequality displayed above and the
	fact that $\left|e^{\beta H}-1\right|=1-e^{\beta H}$ implies the
	desired upper bound on $G_{1}$. 
	
	Next we control $G_{2}$. The assumption $V_{h+1}^{k}(s')\ge V_{h+1}^{\pi}(s')$
	for all $s'\in\cS$ implies that $q_{2}\le q_{3}$ and therefore $G_{2}\ge0$.
	We also have 
	\begin{align*}
	G_{2} & =\frac{1}{(-\beta)}\left(\log\{q_{3}\}-\log\{q_{2}\}\right)\\
	& \le\frac{e^{-\beta H}}{(-\beta)}(q_{3}-q_{2})\\
	& \le e^{\left|\beta\right|H}\left\langle \fmap(s,a),(\Lambda_{h}^{k})^{-1}\sum_{\tau\in[k-1]}\fmap_{h}^{\tau}\left[\E_{s'\sim P_{h}(\cdot\,|\,s_{h}^{\tau},a_{h}^{\tau})}[V_{h+1}^{k}(s')-V_{h+1}^{\pi}(s')]\right]\right\rangle \\
	& =e^{\left|\beta\right|H}\E_{s'\sim P_{h}(\cdot\,|\,s,a)}[V_{h+1}^{k}(s')-V_{h+1}^{\pi}(s')],
	\end{align*}
	where the second step holds by Fact \ref{fact:log_lip} (with $g=e^{\beta H}$,
	$x=q_{3}$, and $y=q_{2}$) and the fact that $q_{3}\ge q_{2}\ge e^{\beta H}$
	(with the last inequality suggested by Lemma \ref{lem:q_lower_bound}),
	and the third step holds by Fact \ref{fact:exp_lip} (with $b=\beta$,
	$x=r_{h}^{\tau}+V_{h+1}^{k}(s)$, and $y=r_{h}^{\tau}+V_{h+1}^{\pi}(s)$)
	and $r_{h}^{\tau}+V_{h+1}^{k}(s)\ge r_{h}^{\tau}+V_{h+1}^{\pi}(s)\ge0$. 
	
	The proof is hence completed.
\end{proof}

The next lemma establishes the dominance of $Q^k_h$ over $Q^*_h$.
\begin{lem}
	\label{lem:lsvi_Q^k >=00003D Q^pi}On the event of Lemma \ref{lem:G1_G2_bound},
	we have $Q_{h}^{k}(s,a)\ge Q_{h}^{\pi}(s,a)$ for all $(k,h,s,a)\in[K]\times[H]\times\cS\times\cA$. 
\end{lem}
\begin{proof}
	For the purpose of the proof, we set $Q_{H+1}^{\pi}(s,a)=Q_{H+1}^{*}(s,a)=0$
	for all $(s,a)\in\cS\times\cA$. We fix a tuple $(k,s,a)\in[K]\times\cS\times\cA$
	and use strong induction on $h$. The base case for $h=H+1$ is satisfied
	since $(Q_{H+1}^{k}-Q_{H+1}^{\pi})(s,a)=0$ for $k\in[K]$ by definition.
	Now we fix an $h\in[H]$ and assume that $0\le(Q_{h+1}^{k}-Q_{h+1}^{*})(s,a)$.
	Moreover, by the induction assumption we have 
	\begin{equation}
	V_{h+1}^{k}(s)=\max_{a'\in\cA}Q_{h+1}^{k}(s,a')\ge\max_{a'\in\cA}Q_{h+1}^{\pi}(s,a')\ge V_{h+1}^{\pi}(s).\label{eq:lsvi_V_dominance}
	\end{equation}
	We also assume that $(s,a)$ satisfies $N_{h}^{k-1}(s,a)\ge1$, since
	otherwise $Q_{h}^{k}(s,a)=H-h+1\ge Q_{h}^{\pi}(s,a)$ and we are done.
	This assumption and Equation \eqref{eq:lsvi_V_dominance} together
	imply $G_{2}\ge0$ by Lemma \ref{lem:G1_G2_bound}. We also have $G_{1}\ge0$
	on the event of Lemma \ref{lem:G1_G2_bound}. Therefore, it follows
	that $(Q_{h}^{k}-Q_{h}^{\pi})(s,a)\ge0$ by Equation \eqref{eq:lsvi_Q_diff_equiv}.
	The induction is completed and so is the proof. 
\end{proof}
Lemma \ref{lem:lsvi_Q^k >=00003D Q^pi} leads to an immediate and
important corollary.
\begin{lem}
	\label{lem:lsvi_V^k_h >=00003D V^pi_h} For any $\delta\in(0,1]$,
	with probability at least $1-\delta/2$, we have $V_{h}^{k}(s)\ge V_{h}^{\pi}(s)$
	for all $(k,h,s)\in[K]\times[H]\times\cS$.
\end{lem}
\begin{proof}
	The result follows from Lemma \ref{lem:lsvi_Q^k >=00003D Q^pi} and
	Equation \eqref{eq:lsvi_V_dominance}. 
\end{proof}

\subsection{Supporting lemmas}

We first present a concentration result.
\begin{lem}
	\label{lem:LSVI_concen}Define 
	\[
	\bar{\cV}_{h+1}\coloneqq\left\{ \bar{V}_{h+1}:\cS\to\real\mid\forall s\in\cS,\ \bar{V}_{h+1}(s)\in[\min\{e^{\beta(H-h)},1\},\max\{e^{\beta(H-h)},1\}]\right\} .
	\]
	There exists a universal constant $c>0$ such that with probability
	$1-\delta$, we have 
	\[
	\left|e^{\beta[r_{h}(s_{h}^{k},a_{h}^{k})+\bar{V}(s_{h+1}^{k})]}-\E_{s'\sim P_{h}(\cdot\mid s_{h}^{k},a_{h}^{k})}e^{\beta[r_{h}(s_{h}^{k},a_{h}^{k})+\bar{V}(s')]}\right|\le c\left|e^{\beta H}-1\right|\sqrt{\frac{S\iota}{N_{h}^{k}(s,a)}}
	\]
	for all $(k,h,s,a)\in[K]\times[H]\times\cS\times\cA$ and all $\bar{V}\in\bar{\cV}_{h+1}$.
\end{lem}
\begin{proof}
	The proof follows the same reasoning as \citet[Lemma 12]{bai2020provable}.
\end{proof}
The next few lemmas help control $\sum_{k\in[K]}(\fmap_{h}^{k})^{\top}(\Lambda_{h}^{k})^{-1}\fmap_{h}^{k}$.

\begin{lem}[{\citet[Lemma D.2]{jin2019provably}}]
	\emph{\label{lem:lsvi_quad_Lambda_pilot_bound}  }Let $\{\fmap_{t}\}_{t\ge0}$
	be a bounded sequence in $\real^{d}$ satisfying $\sup_{t\ge0}\norm[\fmap_{t}]{}\le1$.
	Let $\Lambda_{0}\in\real^{d\times d}$ be a positive definite matrix
	with $\lambda_{\min}(\Lambda_{0})\ge1$. For any $t\ge0$, we define
	$\Lambda_{t}\coloneqq\Lambda_{0}+\sum_{i\in[t]}\fmap_{i}\fmap_{i}^{\top}$.
	Then, we have 
	\[
	\log\left[\frac{\det(\Lambda_{t})}{\det(\Lambda_{0})}\right]\le\sum_{i\in[t]}\fmap_{i}^{\top}\Lambda_{i-1}^{-1}\fmap_{i}\le2\log\left[\frac{\det(\Lambda_{t})}{\det(\Lambda_{0})}\right].
	\]
\end{lem}

\begin{lem}
	\label{lem:lsvi_quad_Lambda_bound}Recall the definitions of $\fmap_{h}^{k}$
	and $\Lambda_{h}^{k}$. For any $h\in[H]$, we have 
	\[
	\sum_{k\in[K]}(\fmap_{h}^{k})^{\top}(\Lambda_{h}^{k})^{-1}\fmap_{h}^{k}\le2d\iota,
	\]
	where $\iota=\log(2dT/\delta)$
\end{lem}
\begin{proof}
	Define $\Gamma_{h}^{k}\coloneqq\lambda\IdMat+\sum_{\tau\in[k-1]}\fmap_{h}^{\tau}(\fmap_{h}^{\tau})^{\top}$
	with $\lambda=1$. It is not hard to see that by the definition of
	$\Lambda_{h}^{k}$ we have $\Lambda_{h}^{k}\preceq\Gamma_{h}^{k}$
	for $h\in[H]$. Since $\lambda_{\min}(\Gamma_{h}^{k})\ge1$ and $\norm[\fmap_{h}^{k}]{}\le1$
	for all $(k,h)\in[K]\times[H]$, by Lemma \ref{lem:lsvi_quad_Lambda_pilot_bound}
	we have for any $h\in[H]$ that 
	\[
	\sum_{k\in[K]}(\fmap_{h}^{k})^{\top}(\Lambda_{h}^{k})^{-1}\fmap_{h}^{k}\le\sum_{k\in[K]}(\fmap_{h}^{k})^{\top}(\Gamma_{h}^{k})^{-1}\fmap_{h}^{k}\le2\log\left[\frac{\det(\Gamma_{h}^{k+1})}{\det(\Gamma_{h}^{1})}\right].
	\]
	Furthermore, note that $\norm[\Gamma_{h}^{k+1}]{}=\norm[\lambda\IdMat+\sum_{\tau\in[k]}\fmap_{h}^{k}(\fmap_{h}^{k})^{\top}]{}\le\lambda+k$.
	This implies 
	\[
	\sum_{k\in[K]}(\fmap_{h}^{k})^{\top}(\Lambda_{h}^{k})^{-1}\fmap_{h}^{k}\le2d\log\left[\frac{\lambda+k}{\lambda}\right]\le2d\iota,
	\]
	as desired.
\end{proof}

\section{Proof of Theorem \ref{thm:regret_LSVI}\label{sec:proof_regret_lsvi}}

Define $\delta_{h}^{k}\coloneqq V_{h}^{k}(s_{h}^{k})-V_{h}^{\pi_{k}}(s_{h}^{k})$,
and $\zeta_{h+1}^{k}\coloneqq\E_{s'\sim P_{h}(\cdot\,|\,s_{h}^{k},a_{h}^{k})}[V_{h+1}^{k}(s')-V_{h+1}^{\pi_{k}}(s')]-\delta_{h+1}^{k}$.
For any $(k,h)\in[K]\times[H]$, we have 
\begin{align}
\delta_{h}^{k} & =(Q_{h}^{k}-Q_{h}^{\pi_{k}})(s_{h}^{k},a_{h}^{k})\nonumber \\
& \le c_{1}\cdot\frac{e^{\left|\beta\right|H}-1}{\left|\beta\right|}\cdot\sqrt{S\iota}\sqrt{\fmap(s_{h}^{k},a_{h}^{k})^{\top}(\Lambda_{h}^{k})^{-1}\fmap(s_{h}^{k},a_{h}^{k})}\nonumber \\
& \quad+e^{\left|\beta\right|H}\cdot\E_{s'\sim P_{h}(\cdot\,|\,s_{h}^{k},a_{h}^{k})}[V_{h+1}^{k}(s')-V_{h+1}^{\pi_{k}}(s')]\nonumber \\
& =c_{1}\cdot\frac{e^{\left|\beta\right|H}-1}{\left|\beta\right|}\cdot\sqrt{S\iota}\sqrt{\fmap(s_{h}^{k},a_{h}^{k})^{\top}(\Lambda_{h}^{k})^{-1}\fmap(s_{h}^{k},a_{h}^{k})}\nonumber \\
& \quad+e^{\left|\beta\right|H}(\delta_{h+1}^{k}+\zeta_{h+1}^{k}).\label{eq:lsvi_delta_recursion}
\end{align}
In the above equation, the first step holds by the construction of
Algorithm \ref{alg:alg_lsvi} and the definition of $V_{h}^{\pi_{k}}$
in Equation \eqref{eq:bellman_primal}; the second step is a consequence
of combining Equation \eqref{eq:lsvi_Q_diff_equiv}  as well as Lemmas
\ref{lem:G1_G2_bound} and \ref{lem:lsvi_V^k_h >=00003D V^pi_h};
the last step follows from the definitions of $\delta_{h}^{k}$
and $\zeta_{h+1}^{k}$. 

Noting that $V_{H+1}^{k}(s)=V_{H+1}^{\pi_{k}}(s)=0$ and the fact
that $\delta_{h+1}^{k}+\zeta_{h+1}^{k}\ge0$ implied by Lemma \ref{lem:lsvi_V^k_h >=00003D V^pi_h},
we can continue by expanding the recursion in Equation \eqref{eq:lsvi_delta_recursion}
and get 
\begin{align}
\delta_{1}^{k} & \le\sum_{h\in[H]}e^{(\left|\beta\right|H)h}\zeta_{h+1}^{k}\nonumber \\
& \quad+c_{1}\cdot\frac{e^{\left|\beta\right|H}-1}{\left|\beta\right|}\cdot\sum_{h\in[H]}e^{(\left|\beta\right|H)(h-1)}\sqrt{S\iota}\sqrt{\fmap(s_{h}^{k},a_{h}^{k})^{\top}(\Lambda_{h}^{k})^{-1}\fmap(s_{h}^{k},a_{h}^{k})}.\label{eq:lsvi_delta1}
\end{align}
Therefore, we have 
\begin{align}
\reg(K) & =\sum_{k\in[K]}\left[(V_{1}^{*}-V_{1}^{\pi_{k}})(s_{1}^{k})\right]\nonumber \\
& \le\sum_{k\in[K]}\delta_{1}^{k}\nonumber \\
& \le e^{\left|\beta\right|H^{2}}\sum_{k\in[K]}\sum_{h\in[H]}\zeta_{h+1}^{k}\nonumber \\
& \quad+c_{1}\cdot\frac{e^{\left|\beta\right|H}-1}{\left|\beta\right|}\cdot e^{\left|\beta\right|H^{2}}\cdot\sqrt{S\iota}\sum_{k\in[K]}\sum_{h\in[H]}\sqrt{\fmap(s_{h}^{k},a_{h}^{k})^{\top}(\Lambda_{h}^{k})^{-1}\fmap(s_{h}^{k},a_{h}^{k})},\label{eq:lsvi_regret_interm}
\end{align}
where the second step holds by Lemma \ref{lem:lsvi_V^k_h >=00003D V^pi_h}
with $\pi$ therein set to the optimal policy, and in the last step
we applied Equation \eqref{eq:lsvi_delta1} along with the Cauchy-Schwarz
inequality. 

We proceed to control the two terms in Equation \eqref{eq:lsvi_regret_interm}.
Since the construction of $V_{h}^{k}$ is independent of the new observation
$s_{h}^{k}$ in episode $k$, we have that $\{\zeta_{h+1}^{k}\}$
is a martingale difference sequence satisfying $\left|\zeta_{h}^{k}\right|\le2H$
for all $(k,h)\in[K]\times[H]$. By the Azuma-Hoeffding inequality,
we have for any $t>0$, 
\[
\P\left(\sum_{k\in[K]}\sum_{h\in[H]}\zeta_{h+1}^{k}\ge t\right)\le\exp\left(-\frac{t^{2}}{2T\cdot H^{2}}\right).
\]
Hence, with probability $1-\delta/2$, there holds 
\begin{equation}
\sum_{k\in[K]}\sum_{h\in[H]}\zeta_{h+1}^{k}\le\sqrt{2TH^{2}\cdot\log(2/\delta)}\le2H\sqrt{T\iota},\label{eq:lsvi_mtg_bound}
\end{equation}
where $\iota=\log(2dT/\delta)$. For the second term in Equation \eqref{eq:lsvi_regret_interm},
we apply Lemma \ref{lem:lsvi_quad_Lambda_bound} and the Cauchy-Schwarz
inequality to obtain 
\begin{align}
& \quad\sum_{k\in[K]}\sum_{h\in[H]}\sqrt{\fmap(s_{h}^{k},a_{h}^{k})^{\top}(\Lambda_{h}^{k})^{-1}\fmap(s_{h}^{k},a_{h}^{k})}\nonumber \\
& \le\sum_{h\in[H]}\sqrt{K}\sqrt{\sum_{k\in[H]}\fmap(s_{h}^{k},a_{h}^{k})^{\top}(\Lambda_{h}^{k})^{-1}\fmap(s_{h}^{k},a_{h}^{k})}\nonumber \\
& \le H\sqrt{2dK\iota}.\label{eq:lsvi_sum_quad_Lambda_bound}
\end{align}
Plugging Equations \eqref{eq:lsvi_mtg_bound} and \eqref{eq:lsvi_sum_quad_Lambda_bound}
back to Equation \eqref{eq:lsvi_regret_interm} yields 
\begin{align*}
\reg(K) & \le e^{\left|\beta\right|H^{2}}\cdot2H\sqrt{T\iota}+c_{1}\cdot\frac{e^{\left|\beta\right|H}-1}{\left|\beta\right|}\cdot e^{\left|\beta\right|H^{2}}\cdot H\sqrt{2dSK\iota^{2}}\\
& \le(c_{1}+2)\cdot\frac{e^{\left|\beta\right|H}-1}{\left|\beta\right|}\cdot e^{\left|\beta\right|H^{2}}\cdot\sqrt{2dHST\iota^{2}},
\end{align*}
where the last step holds since $\frac{e^{\left|\beta\right|H}-1}{\left|\beta\right|}\ge H$.
The proof is completed in view of Fact \ref{fact:exp_factor} and
the identity $d=SA$.

\section{Proof warmup for Theorem \ref{thm:regret_Q_learn}}

Recall the learning rates $\{\alpha_t\}$ defined in Equation \eqref{eq:learn_rate}.
Define the quantities
\begin{equation}
\alpha_{t}^{0}\coloneqq\prod_{j=1}^{t}(1-\alpha_{j}),\qquad\alpha_{t}^{i}\coloneqq\alpha_{i}\prod_{j=i+1}^{t}(1-\alpha_{j})\label{eq:learn_rate_prod}
\end{equation}
for integers $i,t\ge1$. By convention, we set $\alpha_{t}^{0}=1$
and $\sum_{i\in[t]}\alpha_{t}^{i}=0$ if $t=0$, and $\alpha_{t}^{i}=\alpha_{i}$
if $t<i+1$. Define the shorthand $\iota\coloneqq\log(SAT/\delta)$ for $\delta\in(0,1]$.

The following fact describes some key properties of the learning rates $\{\alpha_t\}$.
\begin{fact}
	\label{fact:learn_rate_prop}The following properties hold for $\alpha_{t}^{i}$.
	\begin{enumerate}[label={(\alph*)},ref={\thefact(\alph*)}]
		\item \label{fact:learn_rate_sum_div_sqrti}$\frac{1}{\sqrt{t}}\le\sum_{i\in[t]}\frac{\alpha_{t}^{i}}{\sqrt{i}}\le\frac{2}{\sqrt{t}}$
		for every integer $t\ge1$.
		\item \label{fact:learn_rate_max_L2} $\max_{i\in[t]}\alpha_{t}^{i}\le\frac{2H}{t}$
		and $\sum_{i\in[t]}(\alpha_{t}^{i})^{2}\le\frac{2H}{t}$ for every
		integer $t\ge1$.
		\item \label{fact:learn_rate_infinite_sum}$\sum_{t=i}^{\infty}\alpha_{t}^{i}=1+\frac{1}{H}$
		for every integer $i\ge1$.
		\item \label{fact:learn_rate_binary}$\sum_{i\in[t]}\alpha_{t}^{i}=1$ and
		$\alpha_{t}^{0}=0$ for every integer $t\ge1$, and $\sum_{i\in[t]}\alpha_{t}^{i}=0$
		and $\alpha_{t}^{0}=1$ for $t=0$.
	\end{enumerate}
\end{fact}
\begin{proof}
	The first three facts can be found in \citet[Lemma 4.1]{jin2018q},
	and the last one follows from direct calculation in view of Equation~\eqref{eq:learn_rate_prod}.
\end{proof}

We also present a lemma that controls the deviation of the exponentiated value function from its expectation.
\begin{lem}
	\label{lem:weighted_sum_mtg_bonus_bound} There exists a universal
	constant $c>0$ such that for any $(k,h,s,a)\in[K]\times[H]\times\cS\times\cA$
	and $k_{1},\ldots,k_{t}<k$ with $t=N_{h}^{k}(s,a)$, we have 
	\begin{align*}
	& \quad\left|\frac{1}{\beta}\sum_{i\in[t]}\alpha_{t}^{i}\left[e^{\beta[r_{h}(s,a)+V_{h+1}^{*}(s_{h+1}^{k_{i}})]}-\E_{s'\sim P_{h}(\cdot\,|\,s,a)}e^{\beta[r_{h}(s,a)+V_{h+1}^{*}(s')]}\right]\right|\\
	& \le\frac{c\left|e^{\beta H}-1\right|}{\left|\beta\right|}\sqrt{\frac{H\iota}{t}}.
	\end{align*}
	with probability at least $1-\delta$, and 
	\[
	\frac{1}{\left|\beta\right|}\sum_{i\in[t]}\alpha_{t}^{i}b_{i}\in\left[\frac{c\left|e^{\beta H}-1\right|}{\left|\beta\right|}\sqrt{\frac{H\iota}{t}},\frac{2c\left|e^{\beta H}-1\right|}{\left|\beta\right|}\sqrt{\frac{H\iota}{t}}\right].
	\]
\end{lem}
\begin{proof}
	For any $(k,h,s,a)\in[K]\times[H]\times\cS\times\cA$, define 
	\[
	\psi(i,k,h,s,a)\coloneqq e^{\beta[r_{h}(s,a)+V_{h+1}^{*}(s_{h+1}^{k_{i}})]}-\E_{s'\sim P_{h}(\cdot\,|\,s,a)}e^{\beta[r_{h}(s,a)+V_{h+1}^{*}(s')]}
	\]
	Let us fix a tuple $(k,h,s,a)\in[K]\times[H]\times\cS\times\cA$.
	It can be seen that $\{\indic(k_{i}\le K)\cdot\psi(i,k,h,s,a)\}_{i\in[\tau]}$
	for $\tau\in[K]$ is a martingale difference sequence. By the Azuma-Hoeffding
	inequality and a union bound over $\tau\in[K]$, we have with probability
	at least $1-\delta/(HSA)$, for all $\tau\in[K]$, 
	\begin{align*}
	& \left|\sum_{i\in[\tau]}\alpha_{\tau}^{i}\cdot\indic(k_{i}\le K)\cdot\psi(i,k,h,s,a)\right|\\
	& \le\frac{c\left|e^{\beta H}-1\right|}{2}\sqrt{\iota\sum_{i\in[\tau]}(\alpha_{\tau}^{i})^{2}}\le c\left|e^{\beta H}-1\right|\sqrt{\frac{H\iota}{\tau}}
	\end{align*}
	where $c>0$ is some universal constant, the first step holds since
	$r_{h}(s,a)+V_{h+1}^{*}(s')\in[0,H]$ for $s'\in\cS$, and the last
	step follows from Fact \ref{fact:learn_rate_max_L2}. Since the above
	equation holds for all $\tau\in[K]$, it also holds for $\tau=t=N_{h}^{k}(s,a)\le K$.
	Note that $\indic(k_{i}\le K)=1$ for all $i\in[N_{h}^{k}(s,a)]$.
	Therefore, applying another union bound over $(h,s,a)\in[H]\times\cS\times\cA$,
	we have that the following holds for all $(k,h,s,a)\in[K]\times[H]\times\cS\times\cA$
	and with probability at least $1-\delta$:
	\begin{align}
	\left|\sum_{i\in[t]}\alpha_{\tau}^{i}\cdot\psi(i,k,h,s,a)\right| & \le c\left|e^{\beta H}-1\right|\sqrt{\frac{H\iota}{t}},\label{eq:mds_weighted_bound}
	\end{align}
	where $t=N_{h}^{k}(s,a)$. Using the fact that $r_{h}+V_{h+1}^{*}\in[0,H]$,
	we have 
	\begin{align*}
	& \quad\left|\frac{1}{\beta}\sum_{i\in[t]}\alpha_{t}^{i}\left[\E_{s'\sim\hat{P}_{h}^{k_{i}}(\cdot\,|\,s,a)}e^{\beta[r_{h}(s,a)+V_{h+1}^{*}(s')]}-\E_{s'\sim P_{h}(\cdot\,|\,s,a)}e^{\beta[r_{h}(s,a)+V_{h+1}^{*}(s')]}\right]\right|\\
	& =\left|\frac{1}{\beta}\sum_{i\in[t]}\alpha_{t}^{i}\cdot\psi(i,k,h,s,a)\right|\le\frac{c\left|e^{\beta H}-1\right|}{\left|\beta\right|}\sqrt{\frac{H\iota}{t}}.
	\end{align*}
	
	To prove the result for $\frac{1}{\left|\beta\right|}\sum_{i\in[t]}\alpha_{t}^{i}b_{i}$,
	we recall the definition of $\{b_{t}\}$ in Line \ref{line:qlearn_bonus_def}
	of Algorithm \ref{alg:Q_learn} and compute 
	\begin{align*}
	\frac{1}{\left|\beta\right|}\sum_{i\in[t]}\alpha_{t}^{i}b_{i} & =\frac{c\left|e^{\beta H}-1\right|}{\left|\beta\right|}\sum_{i\in[t]}\alpha_{t}^{i}\sqrt{\frac{H\iota}{i}}\\
	& \in\left[\frac{c\left|e^{\beta H}-1\right|}{\left|\beta\right|}\sqrt{\frac{H\iota}{t}},\frac{2c\left|e^{\beta H}-1\right|}{\left|\beta\right|}\sqrt{\frac{H\iota}{t}}\right]
	\end{align*}
	where the last step holds by Fact \ref{fact:learn_rate_sum_div_sqrti}.
\end{proof}

We fix a tuple $(k,h,s,a)\in[K]\times[H]\times\cS\times\cA$ with
$k_{i}\le k$ being the episode in which $(s,a)$ is taken the $i$-th
time at step $h$. Let us define 
\begin{align*}
q_{1}^{+} & \coloneqq\begin{cases}
\alpha_{t}^{0}e^{\beta(H-h+1)}+\sum_{i\in[t]}\alpha_{t}^{i}\left[e^{\beta[r_{h}(s,a)+V_{h+1}^{k_{i}}(s_{h+1}^{k_{i}})]}+b_{i}\right], & \text{if }\beta>0,\\
\alpha_{t}^{0}e^{\beta(H-h+1)}+\sum_{i\in[t]}\alpha_{t}^{i}\left[e^{\beta[r_{h}(s,a)+V_{h+1}^{k_{i}}(s_{h+1}^{k_{i}})]}-b_{i}\right], & \text{if }\beta<0,
\end{cases}\\
q_{1} & \coloneqq\begin{cases}
\min\{e^{\beta(H-h+1)},q_{1}^{+}\}, & \text{if }\beta>0,\\
\max\{e^{\beta(H-h+1)},q_{1}^{+}\}, & \text{if }\beta<0,
\end{cases}
\end{align*}
and 
\begin{align*}
q_{2}^{+} & \coloneqq\begin{cases}
\alpha_{t}^{0}e^{\beta(H-h+1)}+\sum_{i\in[t]}\alpha_{t}^{i}\left[e^{\beta[r_{h}(s,a)+V_{h+1}^{*}(s_{h+1}^{k_{i}})]}+b_{i}\right], & \text{if }\beta>0,\\
\alpha_{t}^{0}e^{\beta(H-h+1)}+\sum_{i\in[t]}\alpha_{t}^{i}\left[e^{\beta[r_{h}(s,a)+V_{h+1}^{*}(s_{h+1}^{k_{i}})]}-b_{i}\right], & \text{if }\beta<0,
\end{cases}\\
q_{2} & \coloneqq\begin{cases}
\min\{e^{\beta(H-h+1)},q_{2}^{+}\}, & \text{if }\beta>0,\\
\max\{e^{\beta(H-h+1)},q_{2}^{+}\}, & \text{if }\beta<0,
\end{cases}\\
q_{2}' & \coloneqq\alpha_{t}^{0}e^{\beta(H-h+1)}+\sum_{i\in[t]}\alpha_{t}^{i}\left[e^{\beta[r_{h}(s,a)+V_{h+1}^{*}(s_{h+1}^{k_{i}})]}\right],
\end{align*}
and 
\[
q_{3}\coloneqq\alpha_{t}^{0}e^{\beta\cdot Q_{h}^{*}(s,a)}+\sum_{i\in[t]}\alpha_{t}^{i}\left[\E_{s'\sim P_{h}(\cdot\,|\,s,a)}e^{\beta[r_{h}(s,a)+V_{h+1}^{*}(s')]}\right].
\]

We have a simple fact on $q_{2}$ and $q_{2}'$. 
\begin{fact}
	\label{fact:q_2_prime}If $\beta>0$, we have $q_{2}'\le q_{2}$;
	if $\beta<0$, we have $q_{2}'\ge q_{2}$.
\end{fact}
\begin{proof}
	We focus on the case of $\beta>0$. Note that $r_{h}(s,a)+V_{h+1}^{*}(s_{h+1}^{k_{i}})\in[0,H-h+1]$,
	which implies $e^{\beta[r_{h}(s,a)+V_{h+1}^{*}(s_{h+1}^{k_{i}})]}\le e^{\beta(H-h+1)}$.
	We also have $\alpha_{t}^{0},\sum_{i\in[t]}\alpha_{t}^{i}\in\{0,1\}$
	with $\alpha_{t}^{0}+\sum_{i\in[t]}\alpha_{t}^{i}=1$ by Fact \ref{fact:learn_rate_binary}.
	These together imply that $q_{2}'\le e^{\beta H}$ and $q_{2}'-q_{2}^{+}=-\sum_{i\in[t]}\alpha_{t}^{i}b_{i}\le0$
	by definition of $b_{i}$ in Line \ref{line:qlearn_bonus_def} of
	Algorithm \ref{alg:Q_learn}. Therefore, $q_{2}'\le\min\{e^{\beta(H-h+1)},q_{2}^{+}\}=q_{2}$.
	The case of $\beta<0$ can be proved in a similar way and thus omitted.
\end{proof}
Next, we establish a representation of the performance difference
$(Q_{h}^{k}-Q_{h}^{*})(s,a)$ using the quantities $q_{1}$ and $q_{3}$.
\begin{lem}
	\label{lem:recursion_Q} \emph{For any $(k,h,s,a)\in[K]\times[H]\times\cS\times\cA$,
		let $t=N_{h}^{k}(s,a)$ and suppose $(s,a)$ was previously taken
		at step $h$ of episodes $k_{1},\ldots,k_{t}<k$. We have 
		\[
		(Q_{h}^{k}-Q_{h}^{*})(s,a)=\frac{1}{\beta}\log\{q_{1}\}-\frac{1}{\beta}\log\{q_{3}\}.
		\]
	}
\end{lem}
\begin{proof}
	The Bellman optimality equation \eqref{eq:bellman_optimal_primal}
	implies 
	\[
	e^{\beta\cdot Q_{h}^{*}(s,a)}=e^{\beta\cdot r_{h}(s,a)}\left[\E_{s'\sim P_{h}(\cdot\,|\,s,a)}e^{\beta\cdot V_{h+1}^{*}(s')}\right].
	\]
	By Fact \ref{fact:learn_rate_binary}, we have 
	\begin{align*}
	e^{\beta\cdot Q_{h}^{*}(s,a)} & =\alpha_{t}^{0}e^{\beta\cdot Q_{h}^{*}(s,a)}+\sum_{i\in[t]}\alpha_{t}^{i}e^{\beta\cdot r_{h}(s,a)}\left[\E_{s'\sim P_{h}(\cdot\,|\,s,a)}e^{\beta\cdot V_{h+1}^{*}(s')}\right]=q_{3}
	\end{align*}
	for each integer $t\ge0$, and therefore 
	\begin{align}
	Q_{h}^{*}(s,a) & =\frac{1}{\beta}\log\left\{ q_{3}\right\} .\label{eq:Q^*_h(s,a)_rewrite}
	\end{align}
	We finish the proof by combining Equation \eqref{eq:Q^*_h(s,a)_rewrite}
	and the fact that $Q_{h}^{k}(s,a)=\frac{1}{\beta}\log\{q_{1}\}$,
	which follows from Line \ref{line:qlearn_Q_update} of Algorithm \ref{alg:Q_learn}.
\end{proof}
We define the quantities 
\begin{equation}
\begin{aligned}G_{1} & \coloneqq\frac{1}{\beta}\log\{q_{1}\}-\frac{1}{\beta}\log\{q_{2}\},\\
G_{2} & \coloneqq\frac{1}{\beta}\log\{q_{2}\}-\frac{1}{\beta}\log\{q_{3}\},
\end{aligned}
\label{eq:Q_G1_G2_def}
\end{equation}
It is not hard to see that $(Q_{h}^{k}-Q_{h}^{*})(s,a)=G_{1}+G_{2}$
by Lemma \ref{lem:recursion_Q}. The next lemma establishes upper
and lower bounds for $(Q_{h}^{k}-Q_{h}^{*})(s,a)$.

\begin{lem}
	\label{lem:bound_Q^k-Q^*}For all $(k,h,s,a)\in[K]\times[H]\times\cS\times\cA$
	such that $t=N_{h}^{k}(s,a)\ge1$, let 
	\[
	\gamma_{t}\coloneqq2\sum_{i\in[t]}\alpha_{t}^{i}b_{i}\cdot\begin{cases}
	\frac{1}{\left|\beta\right|}, & \text{if }\beta>0,\\
	\frac{e^{-\beta H}}{\left|\beta\right|}, & \text{if }\beta<0,
	\end{cases}
	\]
	and with probability at least $1-\delta$ we have 
	\begin{align*}
	0\le(Q_{h}^{k}-Q_{h}^{*})(s,a) & \le\alpha_{t}^{0}He^{\left|\beta\right|H}+\sum_{i\in[t]}\alpha_{t}^{i}e^{\left|\beta\right|H}\left[V_{h+1}^{k_{i}}(s_{h+1}^{k_{i}})-V_{h+1}^{*}(s_{h+1}^{k_{i}})\right]+2\gamma_{t},
	\end{align*}
	where $k_{1},\ldots,k_{t}<k$ are the episodes in which $(s,a)$ was
	taken at step $h$, and $\gamma_{t}\le\frac{4c(e^{\left|\beta\right|H}-1)}{\left|\beta\right|}\sqrt{\frac{H\iota}{t}}$.
\end{lem}
\begin{proof}
	We prove the lower bound for $(Q_{h}^{k}-Q_{h}^{*})(s,a)$ and then
	use it to prove the upper bound.
	
	\vspace{1em}
	
	\noindent\textbf{Lower bound for $Q^{k}-Q^{*}$.}
	
	\noindent For the purpose of the proof, we set $Q_{H+1}^{k}(s,a)=Q_{H+1}^{*}(s,a)=0$
	for all $(k,s,a)\in[K]\times\cS\times\cA$. We fix a $(s,a)\in\cS\times\cA$
	and use strong induction on $k$ and $h$. Without loss of generality,
	we assume that there exists a $(k,h)$ such that $(s,a)=(s_{h}^{k},a_{h}^{k})$
	(that is, $(s,a)$ has been taken at some point in Algorithm \ref{alg:Q_learn}),
	since otherwise $Q_{h}^{k}(s,a)=H-h+1\ge Q_{h}^{*}(s,a)$ for all
	$(k,h)\in[K]\times[H]$ and we are done. The base case for $k=1$
	and $h=H+1$ is satisfied since $(Q_{H+1}^{k'}-Q_{H+1}^{*})(s,a)=0$
	for $k'\in[K]$ by definition. We fix a $(k,h)\in[K]\times[H]$ and
	assume that $0\le(Q_{h+1}^{k_{i}}-Q_{h+1}^{*})(s,a)$ for each $k_{1},\ldots,k_{t}<k$
	(here $t=N_{h}^{k}(s,a)$). Then we have for $i\in[t]$ that 
	\[
	V_{h+1}^{k_{i}}(s)=\max_{a'\in\cA}Q_{h+1}^{k_{i}}(s,a')\ge\max_{a'\in\cA}Q_{h+1}^{*}(s,a')=V_{h+1}^{*}(s).
	\]
	Recall the quantities $G_{1}$ and $G_{2}$ defined in Equation \eqref{eq:Q_G1_G2_def}.
	The above equation implies $G_{1}\ge0$. We also have $G_{2}\ge0$
	by the fact $Q_{h}^{*}(s,a)\le H$ and on the event of Lemma \ref{lem:weighted_sum_mtg_bonus_bound}.
	Therefore, it follows that $(Q_{h}^{k}-Q_{h}^{*})(s,a)=G_{1}+G_{2}\ge0$.
	The induction is completed and we have proved that $0\le(Q_{h}^{k}-Q_{h}^{*})(s,a)$
	for all $(k,h,s,a)\in[K]\times[H]\times\cS\times\cA$. 
	
	\vspace{1em}
	
	\noindent\textbf{Upper bound for $Q^{k}-Q^{*}$.}
	
	\noindent Let us fix a $(k,h,s,a)\in[K]\times[H]\times\cS\times\cA$.
	Since $0\le(Q_{h}^{k}-Q_{h}^{*})(s,a)$, we have for $i\in[t]$ that
	\[
	V_{h+1}^{k_{i}}(s)=\max_{a'\in\cA}Q_{h+1}^{k_{i}}(s,a')\ge\max_{a'\in\cA}Q_{h+1}^{*}(s,a')=V_{h+1}^{*}(s).
	\]
	
	\textbf{Case $\beta>0$.} We have 
	\begin{align*}
	G_{1} & =\frac{1}{\beta}\log\{q_{1}\}-\frac{1}{\beta}\log\{q_{2}\}\\
	& \le\frac{1}{\beta}(q_{1}-q_{2})\\
	& \le\frac{1}{\beta}(q_{1}^{+}-q_{2}')\\
	& \le\frac{1}{\beta}\sum_{i\in[t]}\alpha_{t}^{i}\left[e^{\beta[r_{h}(s,a)+V_{h+1}^{k_{i}}(s_{h+1}^{k_{i}})]}-e^{\beta[r_{h}(s,a)+V_{h+1}^{*}(s_{h+1}^{k_{i}})]}\right]+\frac{1}{\beta}\sum_{i\in[t]}\alpha_{t}^{i}b_{i}\\
	& \le e^{\left|\beta\right|H}\sum_{i\in[t]}\alpha_{t}^{i}\left[(V_{h+1}^{k_{i}}-V_{h+1}^{*})(s_{h+1}^{k_{i}})\right]+\gamma_{t},
	\end{align*}
	where the second step holds by Fact \ref{fact:log_lip} with $g=1$
	and the fact that $V_{h+1}^{k_{i}}(s)\ge V_{h+1}^{*}(s)$ and by noticing
	that $\alpha_{t}^{0},\sum_{i\in[t]}\alpha_{t}^{i}\in\{0,1\}$ with
	$\alpha_{t}^{0}+\sum_{i\in[t]}\alpha_{t}^{i}=1$ by Fact \ref{fact:learn_rate_binary}
	(so that $q_{1}\ge q_{2}$), the third step holds since by definition
	$q_{1}^{+}\ge q_{1}$ and by Fact \ref{fact:q_2_prime} $q_{2}'\le q_{2}$,
	and the last step holds by Fact \ref{fact:exp_lip} and the fact that
	$H\ge r_{h}(s,a)+V_{h+1}^{k_{i}}(s)\ge r_{h}(s,a)+V_{h+1}^{*}(s)\ge0$.
	For $G_{2}$, we have 
	\begin{align*}
	G_{2} & =\frac{1}{\beta}\log\{q_{2}\}-\frac{1}{\beta}\log\{q_{3}\}\\
	& \le\frac{1}{\beta}(q_{2}-q_{3})\\
	& \le\frac{1}{\beta}(q_{2}^{+}-q_{3})\\
	& =\frac{\alpha_{t}^{0}}{\beta}\left[e^{\beta H}-e^{\beta\cdot Q_{h}^{*}(s,a)}\right]+\frac{1}{\beta}\sum_{i\in[t]}\alpha_{t}^{i}b_{i}\\
	& \quad+\frac{1}{\beta}\sum_{i\in[t]}\alpha_{t}^{i}\left[e^{\beta[r_{h}(s,a)+V_{h+1}^{*}(s_{h+1}^{k_{i}})]}-\E_{s'\sim P_{h}(\cdot\,|\,s,a)}e^{\beta[r_{h}(s,a)+V_{h+1}^{*}(s')]}\right]\\
	& \le\alpha_{t}^{0}He^{\left|\beta\right|H}+\gamma_{t},
	\end{align*}
	In the above, the second step holds by Fact \ref{fact:log_lip} with
	$g=1$ and 
	\[
	\sum_{i\in[t]}\alpha_{t}^{i}b_{i}\ge\left|\sum_{i\in[t]}\alpha_{t}^{i}\left[e^{\beta[r_{h}(s,a)+V_{h+1}^{*}(s_{h+1}^{k_{i}})]}-\E_{s'\sim P_{h}(\cdot\,|\,s,a)}e^{\beta[r_{h}(s,a)+V_{h+1}^{*}(s')]}\right]\right|
	\]
	on the event of Lemma \ref{lem:weighted_sum_mtg_bonus_bound} (so
	that $q_{2}\ge q_{3}$); the third step holds by Fact \ref{fact:q_2_prime};
	the last step holds by Fact \ref{fact:exp_lip} and $Q_{h}^{*}(s,a)\in[0,H]$
	and on the event of Lemma \ref{lem:weighted_sum_mtg_bonus_bound}. 
	
	\textbf{Case $\beta<0$.}\emph{ }We have 
	\begin{align*}
	G_{1} & =\frac{1}{(-\beta)}\log\{q_{2}\}-\frac{1}{(-\beta)}\log\{q_{1}\}\\
	& \le\frac{e^{-\beta H}}{(-\beta)}(q_{2}-q_{1})\\
	& \le\frac{e^{-\beta H}}{(-\beta)}(q_{2}'-q_{1}^{+})\\
	& =\frac{e^{-\beta H}}{(-\beta)}\sum_{i\in[t]}\alpha_{t}^{i}\left[e^{\beta[r_{h}(s,a)+V_{h+1}^{*}(s_{h+1}^{k_{i}})]}-e^{\beta[r_{h}(s,a)+V_{h+1}^{k_{i}}(s_{h+1}^{k_{i}})]}\right]+\frac{e^{-\beta H}}{(-\beta)}\sum_{i\in[t]}\alpha_{t}^{i}b_{i}\\
	& \le e^{\left|\beta\right|H}\sum_{i\in[t]}\alpha_{t}^{i}\left[(V_{h+1}^{k_{i}}-V_{h+1}^{*})(s_{h+1}^{k_{i}})\right]+\gamma_{t},
	\end{align*}
	where the second step holds by Fact \ref{fact:log_lip} with $g=e^{\beta H}$
	and the fact that $V_{h+1}^{k_{i}}(s)\ge V_{h+1}^{*}(s)$ (so that
	$q_{2}\ge q_{1}$), the third step holds since $q_{2}'\ge q_{2}$
	by Fact \ref{fact:q_2_prime} and $q_{1}^{+}\le q_{1}$ by definition,
	and the last step holds by Fact \ref{fact:exp_lip} and the fact that
	$H\ge r_{h}(s,a)+V_{h+1}^{k_{i}}(s)\ge r_{h}(s,a)+V_{h+1}^{*}(s)\ge0$.
	For $G_{2}$, we have 
	\begin{align*}
	G_{2} & =\frac{1}{(-\beta)}\log\{q_{3}\}-\frac{1}{(-\beta)}\log\{q_{2}\}\\
	& \le\frac{e^{-\beta H}}{(-\beta)}(q_{3}-q_{2})\\
	& \le\frac{e^{-\beta H}}{(-\beta)}(q_{3}-q_{2}^{+})\\
	& =\frac{e^{-\beta H}}{(-\beta)}\alpha_{t}^{0}\left[e^{\beta\cdot Q_{h}^{*}(s,a)}-e^{\beta H}\right]+\frac{e^{-\beta H}}{(-\beta)}\sum_{i\in[t]}\alpha_{t}^{i}b_{i}\\
	& \quad+\frac{e^{-\beta H}}{(-\beta)}\sum_{i\in[t]}\alpha_{t}^{i}\left[\E_{s'\sim P_{h}(\cdot\,|\,s,a)}e^{\beta[r_{h}(s,a)+V_{h+1}^{*}(s')]}-e^{\beta[r_{h}(s,a)+V_{h+1}^{*}(s_{h+1}^{k_{i}})]}\right]\\
	& \le e^{-\beta H}\alpha_{t}^{0}\left[H-Q_{h}^{*}(s,a)\right]+\frac{2e^{-\beta H}}{(-\beta)}\sum_{i\in[t]}\alpha_{t}^{i}b_{i}\\
	& \le\alpha_{t}^{0}He^{\left|\beta\right|H}+\gamma_{t}.
	\end{align*}
	where the second step holds by Fact \ref{fact:log_lip} given $q_{3}\ge q_{2}$,
	the second to the last step holds by Fact \ref{fact:exp_lip}, the
	fact that $Q_{h}^{*}(s,a)\le H$ and on the event of Lemma \ref{lem:weighted_sum_mtg_bonus_bound},
	and the last step holds by the definition of $\gamma_{t}$.
	
	Combining the bounds of $G_{1}$ and $G_{2}$ with the identity $(Q_{h}^{k}-Q_{h}^{*})(s,a)=G_{1}+G_{2}$
	yields the upper bound for $(Q_{h}^{k}-Q_{h}^{*})(s,a)$. The proof
	is completed in view of Lemma \ref{lem:weighted_sum_mtg_bonus_bound}
	and the definition of $\gamma_{t}$ that imply
	\[
	\gamma_{t}\le\frac{4c(e^{\left|\beta\right|H}-1)}{\left|\beta\right|}\sqrt{\frac{H\iota}{t}}.
	\]
\end{proof}

\section{Proof of Theorem \ref{thm:regret_Q_learn}\label{sec:proof_regret_Q_learn}}

We first introduce some notations. Let $\cG$ be a discrete space.
Define the shorthand 
\begin{equation}
\lse_{\beta}(P,f)\coloneqq\frac{1}{\tmp}\log\left\{ \E_{x\sim P}\left[\exp\left(\beta\cdot f(x)\right)\right]\right\} ,\label{eq:lse_def}
\end{equation}
for a probability distribution $P$ supported on $\cG$ and function
$f:\cG\to\real$. We record a useful lemma that shows $\lse_{\beta}(\cdot,\cdot)$
is Lipschitz continuous in the second argument. 
\begin{lem}
	\label{lem:lse_lip_f}Let $\cG$ be a discrete space and $\bar{f}\ge0$
	be a non-negative number. Let the functions $f,f':\real^{d}\mapsto[0,\bar{f}]$
	be such that $f(x)\ge f'(x)$ for all $x\in\real^{d}$. Also let $P$
	be a probability distribution supported on $\cG$. We have 
	\[
	\lse_{\beta}(P,f)-\lse_{\beta}(P,f')\le e^{\left|\beta\right|\bar{f}}\cdot\E_{x\sim P}[f(x)-f'(x)].
	\]
\end{lem}
The proof is given in Appendix \ref{sec:proof_lse_lip_f}.

Define $\hat{P}_{h}^{k}(\cdot\,|\,s,a)$ to be the delta function
centered at $s_{h+1}^{k}$ for all $(k,h,s,a)\in[K]\times[H]\times\cS\times\cA$,
and this means $\E_{s'\sim\hat{P}_{h}^{k}(\cdot\,|\,s,a)}[f(s')]=f(s_{h+1}^{k})$
for any function $f:\cS\to\real$. Also define 
\[
\delta_{h}^{k}\coloneqq(V_{h}^{k}-V_{h}^{\pi_{k}})(s_{h}^{k})\quad\text{and}\quad\phi_{h}^{k}\coloneqq(V_{h}^{k}-V_{h}^{*})(s_{h}^{k}).
\]
Also define 
\[
\xi_{h+1}^{k}\coloneqq[(P_{h}-\hat{P}_{h}^{k})(V_{h+1}^{*}-V_{h+1}^{\pi_{k}})](s_{h}^{k},a_{h}^{k}).
\]
Note that For each $(k,h)\in[K]\times[H]$, we have 
\begin{align}
\delta_{h}^{k} & =(Q_{h}^{k}-Q_{h}^{\pi_{k}})(s_{h}^{k},a_{h}^{k})\nonumber \\
& =(Q_{h}^{k}-Q_{h}^{*})(s_{h}^{k},a_{h}^{k})+(Q_{h}^{*}-Q_{h}^{\pi_{k}})(s_{h}^{k},a_{h}^{k})\nonumber \\
& \le\alpha_{t}^{0}He^{\left|\beta\right|H}+\sum_{i\in[t]}\alpha_{t}^{i}e^{\left|\beta\right|H}\phi_{h+1}^{k_{i}}+2\gamma_{t}\nonumber \\
& \quad+[\lse(P_{h}(\cdot\,|\,s_{h}^{k},a_{h}^{k}),V_{h+1}^{*})-\lse(P_{h}(\cdot\,|\,s_{h}^{k},a_{h}^{k}),V_{h+1}^{\pi_{k}})]\nonumber \\
& \le\alpha_{t}^{0}He^{\left|\beta\right|H}+\sum_{i\in[t]}\alpha_{t}^{i}e^{\left|\beta\right|H}\phi_{h+1}^{k_{i}}+2\gamma_{t}+e^{\left|\beta\right|H}[P_{h}(V_{h+1}^{*}-V_{h+1}^{\pi_{k}})](s_{h}^{k},a_{h}^{k})\nonumber \\
& =\alpha_{t}^{0}He^{\left|\beta\right|H}+\sum_{i\in[t]}\alpha_{t}^{i}e^{\left|\beta\right|H}\phi_{h+1}^{k_{i}}+2\gamma_{t}+e^{\left|\beta\right|H}(\delta_{h+1}^{k}-\phi_{h+1}^{k}+\xi_{h+1}^{k}),\label{eq:qlearn_regret_decomp}
\end{align}
where the third step holds by Lemma \ref{lem:bound_Q^k-Q^*} and the
Bellman equations \eqref{eq:bellman_primal} and \eqref{eq:bellman_optimal_primal},
the fourth step holds by Lemma \ref{lem:lse_lip_f} and the fact that
$0\le V_{h+1}^{\pi_{k}}(s)\le V_{h+1}^{*}(s)\le H$ for all $s\in\cS$,
and the last step follows by defintion that $\delta_{h+1}^{k}-\phi_{h+1}^{k}=(V_{h+1}^{*}-V_{h+1}^{\pi_{k}})(s_{h+1}^{k})=[\hat{P}_{h}^{k}(V_{h+1}^{*}-V_{h+1}^{\pi_{k}})](s_{h}^{k},a_{h}^{k})$
and the definition of $\xi_{h+1}^{k}$.

We now compute $\sum_{k\in[K]}\delta_{h}^{k}$ for a fixed $h\in[H]$.
Denote by $n_{h}^{k}\coloneqq N_{h}^{k}(s_{h}^{k},a_{h}^{k})$ and
we have
\[
\sum_{k\in[K]}\alpha_{n_{h}^{k}}^{0}He^{\left|\beta\right|H}=He^{\left|\beta\right|H}\sum_{k\in[K]}\indic\{n_{h}^{k}=0\}\le He^{\left|\beta\right|H}SA.
\]
Then we turn to control the second term in Equation \eqref{eq:qlearn_regret_decomp}
summed over $k\in[K]$, that is, 
\[
\sum_{k\in[K]}\sum_{i\in[t]}\alpha_{t}^{i}e^{\left|\beta\right|H}\phi_{h+1}^{k_{i}}=e^{\left|\beta\right|H}\sum_{k\in[K]}\sum_{i\in[n_{h}^{k}]}\alpha_{n_{h}^{k}}^{i}\phi_{h+1}^{k_{i}(s_{h}^{k},a_{h}^{k})},
\]
where $k_{i}(s_{h}^{k},a_{h}^{k})$ denotes the episode in which $(s_{h}^{k},a_{h}^{k})$
was taken at step $h$ for the $i$-th time. We re-group the above
summation in a different way. For every $k'\in[K]$, the term $\phi_{h+1}^{k'}$
appears in the summand with $k>k'$ if and only if $(s_{h}^{k},a_{h}^{k})=(s_{h}^{k'},a_{h}^{k'})$.
The first time it appears we have $n_{h}^{k}=n_{h}^{k'}+1$, the second
time it appears we have $n_{h}^{k}=n_{h}^{k'}+2$, and etc. Therefore,
\[
e^{\left|\beta\right|H}\sum_{k\in[K]}\sum_{i\in[n_{h}^{k}]}\alpha_{n_{h}^{k}}^{i}\phi_{h+1}^{k_{i}(s_{h}^{k},a_{h}^{k})}\le e^{\left|\beta\right|H}\sum_{k'\in[K]}\phi_{h+1}^{k'}\sum_{t\ge n_{h}^{k'}+1}\alpha_{t}^{n_{h}^{k'}}\le e^{\left|\beta\right|H}\left(1+\frac{1}{H}\right)\sum_{k'\in[K]}\phi_{h+1}^{k'},
\]
where the last step follows Fact \ref{fact:learn_rate_infinite_sum}.
Collecting the above results and plugging them into Equation \eqref{eq:qlearn_regret_decomp},
we have 
\begin{align}
\sum_{k\in[K]}\delta_{h}^{k} & \le He^{\left|\beta\right|H}SA+e^{\left|\beta\right|H}\left(1+\frac{1}{H}\right)\sum_{k\in[K]}\phi_{h+1}^{k}\nonumber \\
& \quad+e^{\left|\beta\right|H}\sum_{k\in[K]}(\delta_{h+1}^{k}-\phi_{h+1}^{k})+\sum_{k\in[K]}(2\gamma_{n_{h}^{k}}+e^{\left|\beta\right|H}\xi_{h+1}^{k})\nonumber \\
& \le He^{\left|\beta\right|H}SA+e^{\left|\beta\right|H}\left(1+\frac{1}{H}\right)\sum_{k\in[K]}\delta_{h+1}^{k}\nonumber \\
& \quad+\sum_{k\in[K]}(2\gamma_{n_{h}^{k}}+e^{\left|\beta\right|H}\xi_{h+1}^{k}),\label{eq:qlearn_regret_interm}
\end{align}
where the last step holds since $\delta_{h+1}^{k}\ge\phi_{h+1}^{k}$
(due to the fact that $V_{h+1}^{*}(s)\ge V_{h+1}^{\pi_{k}}(s)$ for
all $x\in\cS$). Since it holds that 
\[
\left[e^{\left|\beta\right|H}\left(1+\frac{1}{H}\right)\right]^{H}\le e^{\left|\beta\right|H^{2}+1},
\]
we can expand the quantity $\sum_{k\in[K]}\delta_{1}^{k}$ recursively
in the form of Equation \eqref{eq:qlearn_regret_interm}, apply Holder's
inequality and use the fact that $\delta_{H+1}^{k}=0$ to get 
\begin{equation}
\sum_{k\in[K]}\delta_{1}^{k}\le e^{\left|\beta\right|H^{2}+1}\left[H^{2}e^{\left|\beta\right|H}SA+\sum_{h\in[H]}\sum_{k\in[K]}(2\gamma_{n_{h}^{k}}+e^{\left|\beta\right|H}\xi_{h+1}^{k})\right].\label{eq:qlearn_regret_interm_unrolled}
\end{equation}
By the pigeonhole principle, for any $h\in[H]$ we have
\begin{align}
\sum_{k\in[K]}\gamma_{n_{h}^{k}} & \lesssim\frac{e^{\left|\beta\right|H}-1}{\left|\beta\right|}\sum_{k\in[K]}\sqrt{\frac{H\iota}{n_{h}^{k}}}\nonumber \\
& =\frac{e^{\left|\beta\right|H}-1}{\left|\beta\right|}\sum_{(s,a)\in\cS\times\cA}\sum_{n\in[N_{h}^{K}(s,a)]}\sqrt{\frac{H\iota}{n}}\nonumber \\
& \lesssim\frac{e^{\left|\beta\right|H}-1}{\left|\beta\right|}\sqrt{HSAK\iota}\nonumber \\
& =\frac{e^{\left|\beta\right|H}-1}{\left|\beta\right|}\sqrt{SAT\iota},\label{eq:qlearn_bonus_bound}
\end{align}
where the third step holds since $\sum_{(s,a)\in\cS\times\cA}N_{h}^{K}(s,a)=K$
and the RHS of the second step is maximized when $N_{h}^{K}(s,a)=K/(SA)$
for all $(s,a)\in\cS\times\cA$. Finally, the Azuma-Hoeffding inequality
implies that with probability at least $1-\delta$, we have 
\begin{equation}
\left|\sum_{h\in[H]}\sum_{k\in[K]}\xi_{h+1}^{k}\right|\lesssim H\sqrt{T\iota}.\label{eq:qlearn_mtg_bound}
\end{equation}
Putting together Equations \eqref{eq:qlearn_bonus_bound} and \eqref{eq:qlearn_mtg_bound}
and plugging them into \eqref{eq:qlearn_regret_interm_unrolled},
we have 
\begin{align*}
\sum_{k\in[K]}\delta_{1}^{k} & \lesssim e^{\left|\beta\right|(H^{2}+H)}\cdot H^{2}SA\\
& \quad+e^{\left|\beta\right|H^{2}}\cdot\frac{e^{\left|\beta\right|H}-1}{\left|\beta\right|}\sqrt{H^{2}SAT\iota}\\
& \quad+e^{\left|\beta\right|(H^{2}+H)}\cdot H\sqrt{T\iota}.\\
& \le e^{\left|\beta\right|(H^{2}+H)}\cdot H^{2}SA\\
& \quad+e^{\left|\beta\right|(H^{2}+H)}\cdot\frac{e^{\left|\beta\right|H}-1}{\left|\beta\right|}\sqrt{H^{2}SAT\iota}
\end{align*}
The proof is completed in view of Fact \ref{fact:exp_factor} and
when $T$ is sufficiently large.

\subsection{Proof of Lemma \ref{lem:lse_lip_f}\label{sec:proof_lse_lip_f}}

We have the following two cases.

\textbf{Case $\beta>0$.} We have 
\begin{align*}
\lse_{\beta}(P,f)-\lse_{\beta}(P,f') & \le\frac{1}{\beta}\E_{x\sim P}\left[e^{\beta\cdot f(x)}-e^{\beta\cdot f'(x)}\right]\\
& \le\frac{1}{\beta}\E_{x\sim P}\left[\beta e^{\beta\bar{f}}(f(x)-f'(x))\right]\\
& =e^{\beta\bar{f}}\cdot\E_{x\sim P}[f(x)-f'(x)],
\end{align*}
where the first step holds by Fact \ref{fact:log_lip} with $g=1$
and the fact that $e^{\beta\cdot f(x)}\ge e^{\beta\cdot f'(x)}\ge1$,
and the second holds by Fact \ref{fact:exp_lip} with $u=\bar{f}$
and the fact that $f(x)\ge f'(x)$.

\textbf{Case $\beta<0$.} We have 
\begin{align*}
\lse_{\beta}(P,f)-\lse_{\beta}(P,f') & =-\left[\lse_{\beta}(P,f')-\lse_{\beta}(P,f)\right]\\
& \le\frac{\exp(-\beta\bar{f})}{(-\beta)}\E_{x\sim P}\left[\exp(\beta\cdot f'(x))-\exp(\beta\cdot f(x))\right]\\
& \le\frac{\exp(-\beta\bar{f})}{(-\beta)}\E_{x\sim P}\left[(-\beta)(f(x)-f'(x))\right]\\
& =\exp(-\beta\bar{f})\cdot\E_{x\sim P}[f(x)-f'(x)],
\end{align*}
where the second step holds by Fact \ref{fact:log_lip} with $g=e^{\beta\bar{f}}$
given that $x\in[e^{\beta\bar{f}},1]$, and the third step holds by
Fact \ref{fact:exp_lip} and the fact $1\ge e^{\beta\cdot f'(x)}\ge e^{\beta\cdot f(x)}>0$.

\section{Proof of Theorem \ref{thm:regret_lower_bound}\label{sec:proof_regret_lower_bound}}

For each $ \rho \in [0,1] $, let $ \text{Ber}(\rho) $ denote the Bernoulli distribution with parameter $ \rho $. Before diving into the proof, let us record two important results. 
\begin{lem}
	\label{lem:subopt_plays}Let $p,p'\in(0,1)$ and $p>p'$. Define $D\coloneqq D_{\textup{KL}}(\textup{Ber}(p')\|\textup{Ber}(p))$
	to be the KL divergence between $\textup{Ber}(p')$ and $\textup{Ber}(p)$.
	For any policy $\pi$ and a positive integer $K$, let $K_{0}\coloneqq K_{0}(K,\text{\ensuremath{\pi}})$
	be the number of times that the sub-optimal arm is pulled in the $K$-round
	two-arm bandit problem (with $\textup{Ber}(p')$ and $\textup{Ber}(p)$
	being the two arms) when executing policy $\pi$. When $K$ is sufficiently
	large, we have 
	\[
	\E K_{0}\gtrsim\frac{\log K}{D}.
	\]
\end{lem}
\begin{proof}
	This is an intermediate result in the proof of \citet[Theorem 16.2]{lattimore2018bandit}.
\end{proof}
\begin{lem}
	\label{lem:Bern_KL}Let $p,p'\in(0,1)$ be such that $p>p'$. We have
	$D_{\textup{KL}}(\textup{Ber}(p')\|\textup{Ber}(p))\le\frac{(p-p')^{2}}{p(1-p)}$. 
\end{lem}
The proof is provided in Appendix \ref{sec:proof_Bern_KL}.
We consider two cases: $\beta>0$ and $\beta<0$.

\subsection{Case $\beta>0$}

Consider a two-arm bandit problem with $K$ rounds, where the reward for pulling arm
$i\in\{1,2\}$ is given by the scaled $ \text{Ber}(p_i) $ random variable 
\[
X_{i}=\begin{cases}
H & \text{w.p. }p_{i},\\
0 & \text{w.p. }1-p_{i},
\end{cases}
\]
where $H\ge1$ specifies the range of the reward, and the parameters $p_{1}>p_{2}$ are to be specified later. Let $\Delta:=p_{1}-p_{2}>0$.

By Lemma \ref{lem:Bern_KL}, we have 
\begin{equation}
D_{\text{KL}}(X_{2}\|X_{1})\ge\frac{\Delta^{2}}{p_{1}(1-p_{1})}.\label{eq:KL_bound}
\end{equation}
It then follows from Lemma \ref{lem:subopt_plays} that 
\begin{equation}
\E K_{0}\gtrsim\frac{\log K\cdot p_{1}(1-p_{1})}{\Delta^{2}}\label{eq:EK0}
\end{equation}
Let us choose 
\[
\Delta =  C\sqrt{\frac{\log K\cdot p_{1}(1-p_{1})}{K}}
\]
for an universal constant $ C >0$. 
Note that under this choice we have $K\ge\E K_{0}$ as should be expected.
Now, we set $p_{2}=e^{-\beta H}$. Since $p_{1}(1-p_{1})\le\frac{1}{4}$,
we have $\Delta\lesssim\sqrt{\frac{\log K}{K}}$. By choosing $K$
and $H$ large enough, we can ensure $\Delta\le e^{-\beta H}$ and
$p_{1}=p_{2}+\Delta\le\frac{3}{4}$.

Define $X_{i}^{k}$ to be the outcome of arm $X_{i}$ (if pulled)
in round $k$, and $Y^{k}$ to be the outcome of the arm actually
pulled in round $k$. Then, conditional on $K_{0}$, we have 
\begin{align}
\textup{Regret}(K) & =\frac{1}{\beta}\log\left[\mathbb{E}\exp\left(\beta\sum_{k\in[K]}X_{1}^{k}\right)\right]-\frac{1}{\beta}\log\left[\mathbb{E}\exp\left(\beta\sum_{k\in[K]}Y^{k}\right)\right]\nonumber \\
& \overset{(i)}{=}\frac{1}{\beta}\log\left[\prod_{k=1}^{K}\mathbb{E}\exp\left(\beta X_{1}^{k}\right)\right]-\frac{1}{\beta}\log\left[\prod_{k=1}^{K}\mathbb{E}\exp\left(\beta Y^{k}\right)\right]\nonumber \\
& \ge\frac{1}{\beta}\log\left[\prod_{k=1}^{K}\mathbb{E}\exp\left(\beta X_{1}^{k}\right)\right]-\frac{1}{\beta}\log\left[\prod_{k=1}^{K}\mathbb{E}\exp\left(\beta X_{2}^{k}\right)\right]\nonumber \\
& =\frac{K}{\beta}\log\left[\mathbb{E}\exp\left(\beta X_{1}\right)\right]-\frac{K}{\beta}\log\left[\mathbb{E}\exp\left(\beta X_{2}\right)\right]\nonumber \\
& \ge\frac{K_{0}}{\beta}\log\left[\mathbb{E}\exp\left(\beta X_{1}\right)\right]-\frac{K_{0}}{\beta}\log\left[\mathbb{E}\exp\left(\beta X_{2}\right)\right],\label{eq:regret_K0}
\end{align}
where step $(i)$ holds because of the independence among $\{X_{1}^{k}\}$
and independence among $\{Y^{k}\}$. Taking expectation over $K_{0}$
on both sides of Equation (\ref{eq:regret_K0}), we have 
\begin{align*}
\E [\textup{Regret}(K)] & \ge\frac{\E K_{0}}{\beta}\left(\log\mathbb{E}e^{\beta X_{1}}-\log\mathbb{E}e^{\beta X_{2}}\right)\\
& =\frac{\E K_{0}}{\beta}\log\left(\frac{p_{1}e^{\beta H}+(1-p_{1})}{p_{2}e^{\beta H}+(1-p_{2})}\right)\\
& =\frac{\E K_{0}}{\beta}\log\left(1+\frac{\Delta(e^{\beta H}-1)}{p_{2}e^{\beta H}+(1-p_{2})}\right)\\
& \overset{(i)}{\ge}\frac{\E K_{0}}{\beta}\log\left(1+\frac{\Delta(e^{\beta H}-1)}{1+1}\right)\\
& \overset{(ii)}{\ge}\frac{\E K_{0}}{\beta}\cdot\frac{1}{4}\Delta(e^{\beta H}-1)\\
& \overset{(iii)}{\gtrsim}\frac{1}{\beta}\cdot\frac{\log K\cdot p_{1}(1-p_{1})}{\Delta}\cdot(e^{\beta H}-1)\\
& \gtrsim\frac{1}{\beta}\cdot\sqrt{K\log K\cdot p_{1}(1-p_{1})}\cdot(e^{\beta H}-1)\\
& \overset{(iv)}{\gtrsim}\frac{1}{\beta}\cdot\sqrt{K\log K}\cdot(e^{\beta H/2}-1)\\
& \gtrsim\frac{1}{\beta}\cdot\sqrt{T\log T}\cdot(e^{\beta H/2}-1),
\end{align*}
where step $(i)$ holds since $p_{2}=e^{-\beta H}$, step $(ii)$
holds since $\Delta\le e^{-\beta H}$ and $\log(1+x)\ge\frac{x}{2}$
for $x\in[0,1]$, step $(iii)$ holds by Equation (\ref{eq:EK0}),
 step $(iv)$ holds since $e^{-\beta H}=p_{2}\le p_{1}\le\frac{3}{4}$
by construction, and the last step holds since $\frac{1}{\beta}(e^{\beta H/2}-1) \gtrsim H$ implied by Fact \ref{fact:exp_growth} below.
\begin{fact} \label{fact:exp_growth}
	For any $G>0$, the function
	\[
	f_{G}(x) = \frac{e^{Gx}-1}{x}, \quad x > 0
	\]
	is increasing and satisfies $\lim_{x\to0} f_{G}(x) = G$.
\end{fact}

Finally, note that the aforementioned $K$-round two-arm bandit model 
is a special case of an $K$-episode $(H+2)$-horizon MDP with the
per-step reward in $[0,1]$, illustrated in Figure~\ref{fig:bandit_to_MDP}. The MDP is equipped with $\cA=\{a_{1},a_{2}\}$,
$\cS=\{s_{1},s_{2},s_{3}\}$, where state $s_{1}$ is the initial
state, and states $s_{2}$ and $s_{3}$ are absorbing regardless of
actions taken. The states satisfy that $r_{h}(s_{2},a)=1,r_{h}(s_{1},a)=r_{h}(s_{3},a)=0$
for all $h\in[H+2]$ and $a\in\cA$. At the initial state $s_{1}$,
we may choose to take action $a_{1}$ or $a_{2}$. If $a_{1}$ is
taken at state $s_{1}$, then we transition to $s_{2}$ with probability
$p_{1}$ and to $s_{3}$ with probability $1-p_{1}$. If $a_{2}$
is taken at state $s_{1}$, then we transition to $s_{2}$ with probability
$p_{2}$ and to $s_{3}$ with probability $1-p_{2}$. 
\begin{figure}
	\centering
	\includegraphics[scale=.8]{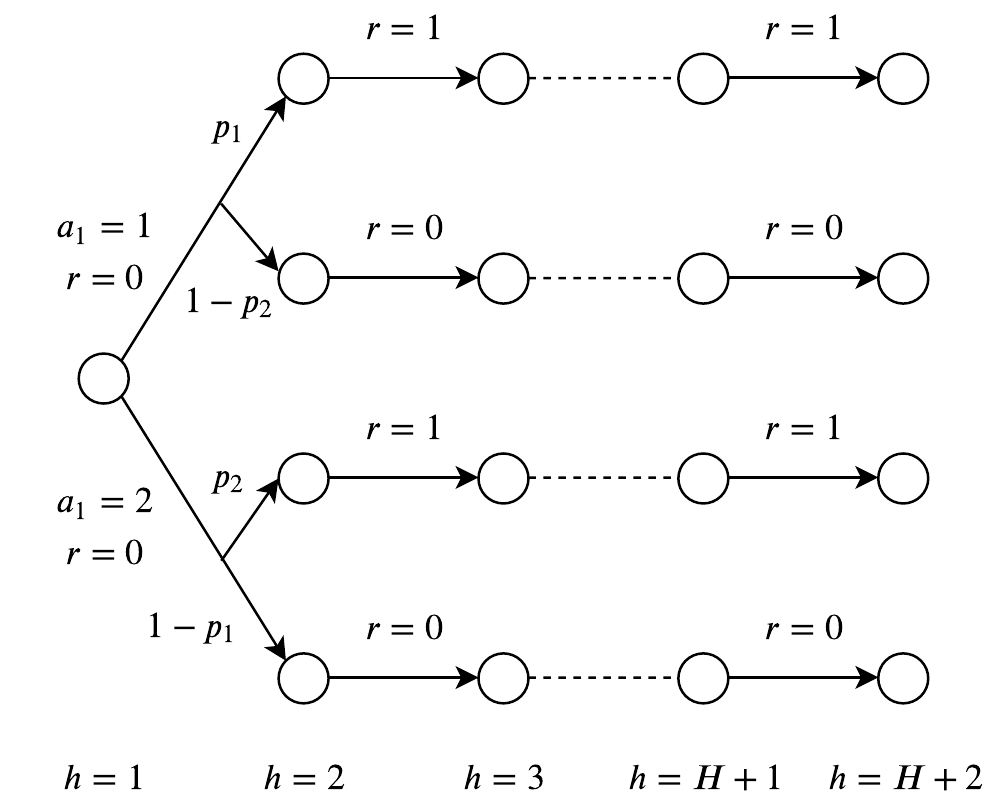}
	\caption{From bandit model to MDP.}
	\label{fig:bandit_to_MDP}
\end{figure}
\subsection{Case $\beta<0$}

The proof of the case $\beta<0$ is similar to that of the case $\beta>0$.
For $\beta<0$, consider a 2-arm bandit model with $K$ rounds, where
the reward for pulling arm
$i\in\{1,2\}$ is given by the scaled $\text{Ber}(1-p_{i})  $ random variable 
\[
X_{i}=\begin{cases}
0 & \text{w.p. }p_{i},\\
H & \text{w.p. }1-p_{i}.
\end{cases}
\]
Let $p_{2}=e^{\beta H}>p_{1}$ and $\Delta\coloneqq p_{1}-p_{2}<0$.
Note that Equations (\ref{eq:KL_bound}) and (\ref{eq:EK0}) remain
valid (by invoking Lemmas \ref{lem:Bern_KL} and \ref{lem:subopt_plays}
with $p=1-p_{1}$ and $p'=1-p_{2}$). Therefore, we choose 
\[
\Delta=-C\sqrt{\frac{\log K\cdot p_{1}(1-p_{1})}{K}}
\]
for some universal constant $C>0$. Since $p_{1}(1-p_{1})\le\frac{1}{4}$,
we have $\Delta\gtrsim-\sqrt{\frac{\log K}{K}}$. By choosing $H$
large enough, we have $1-p_{1}\ge1-p_{2}=1-e^{\beta H}\ge\frac{1}{4}$.
And by choosing $K$ large enough, we can ensure $\Delta\ge-\frac{1}{2}e^{\beta H}$
so that $p_{1}=p_{2}+\Delta\ge\frac{1}{2}e^{\beta H}$. 

Taking the expectation over $K_{0}$ on both sides of Equation (\ref{eq:regret_K0}),
we have 
\begin{align*}
\E [\textup{Regret} (K)] & =\frac{\E K_{0}}{\beta}\left(\log\mathbb{E}e^{\beta X_{1}}-\log\mathbb{E}e^{\beta X_{2}}\right)\\
& =\frac{\E K_{0}}{\beta}\log\left(\frac{(1-p_{1})e^{\beta H}+p_{1}}{(1-p_{2})e^{\beta H}+p_{2}}\right)\\
& =\frac{\E K_{0}}{\beta}\log\left(1+\frac{\Delta(1-e^{\beta H})}{(1-p_{2})e^{\beta H}+p_{2}}\right).\\
& \overset{(i)}{\ge}\frac{\E K_{0}}{\beta}\cdot\frac{\Delta(1-e^{\beta H})}{(1-p_{2})e^{\beta H}+p_{2}}\\
& \overset{(ii)}{\ge}\frac{\E K_{0}}{\beta}\cdot\frac{\Delta(1-e^{\beta H})}{2e^{\beta H}}\\
& \overset{(iii)}{\gtrsim}\frac{1}{(-\beta)}\cdot\frac{\log K\cdot p_{1}(1-p_{1})}{(-\Delta)}\cdot(e^{-\beta H}-1)\\
& =\frac{1}{(-\beta)}\cdot\sqrt{K\log K\cdot p_{1}(1-p_{1})}\cdot(e^{-\beta H}-1)\\
& \overset{(iv)}{=}\frac{1}{(-\beta)}\cdot\sqrt{K\log K}\cdot(e^{-\beta H/2}-1)\\
& \gtrsim\frac{1}{(-\beta)}\cdot\sqrt{T\log T}\cdot(e^{-\beta H/2}-1).
\end{align*}
In the above, step $(i)$ holds since $\beta<0$ and $\log(1+x)\le x$
for all $x>-1$; step $(ii)$ holds since $p_{2}=e^{\beta H}$ and
$\Delta,\beta<0$; step $(iii)$ holds by Equation (\ref{eq:EK0});
step $(iv)$ holds since $p_{1}\ge\frac{1}{2}e^{\beta H}$ and $1-p_{1}\ge\frac{1}{4}$
by construction; 
and the last step holds since $\frac{1}{(-\beta)}(e^{-\beta H/2}-1) \gtrsim H$ implied by Fact \ref{fact:exp_growth}.

It is not hard to see that the two-arm bandit model discussed above 
is also a special case of an $K$-episode $(H+2)$-horizon MDP with the
per-step reward in $[0,1]$, similar to the case $\beta>0$.

\subsection{Proof of Lemma \ref{lem:Bern_KL}\label{sec:proof_Bern_KL}}

Recall that $\Delta\coloneqq p-p'$. The KL divergence can be upper
bounded as follows:
\begin{align*}
D_{\text{KL}}(\text{Ber}(p')\Vert\text{Ber}(p)) & =p'\log\left(\frac{p'}{p}\right)+(1-p')\log\left(\frac{1-p'}{1-p}\right)\\
& =p'\log\left(1+\frac{p'-p}{p}\right)+(1-p')\log\left(1+\frac{p-p'}{1-p}\right)\\
& \overset{(i)}{\le}p'\cdot\frac{p'-p}{p}+(1-p')\cdot\frac{p-p'}{1-p}\\
& =(\Delta-p)\cdot\frac{\Delta}{p}+(1-p+\Delta)\cdot\frac{\Delta}{1-p}\\
& =\frac{\Delta^{2}}{p}+\frac{\Delta^{2}}{1-p}\\
& =\frac{\Delta^{2}}{p(1-p)},
\end{align*}
where step $(i)$ holds since $\log(1+x)\le x$ for all $x>-1$. The
proof is completed.

\end{document}